%% file: main.tex
\documentclass{article}

\PassOptionsToPackage{numbers, compress}{natbib}

\usepackage[preprint]{neurips_2026}

\usepackage[utf8]{inputenc}
\usepackage[T1]{fontenc}
\usepackage{hyperref}
\usepackage{url}
\usepackage{booktabs}
\usepackage{tabularx}
\usepackage[most]{tcolorbox}
\usepackage{placeins}

\newtcblisting{systempromptbox}{
  enhanced, breakable,
  colback=white,
  colframe=gray!75!black,
  boxrule=0.6pt,
  arc=2mm,
  left=6pt, right=6pt, top=4pt, bottom=4pt,
  fonttitle=\bfseries\small,
  coltitle=white,
  colbacktitle=gray!75!black,
  title=System Prompt,
  listing only,
  listing options={
    basicstyle=\ttfamily\scriptsize,
    breaklines=true,
    breakatwhitespace=false,
    columns=fullflexible,
    keepspaces=true,
    aboveskip=0pt,
    belowskip=0pt,
  }
}

\newtcblisting{userpromptbox}{
  enhanced, breakable,
  colback=white,
  colframe=gray!55!black,
  boxrule=0.6pt,
  arc=2mm,
  left=6pt, right=6pt, top=4pt, bottom=4pt,
  fonttitle=\bfseries\small,
  coltitle=white,
  colbacktitle=gray!55!black,
  title=User Prompt,
  listing only,
  listing options={
    basicstyle=\ttfamily\scriptsize,
    breaklines=true,
    breakatwhitespace=false,
    columns=fullflexible,
    keepspaces=true,
    aboveskip=0pt,
    belowskip=0pt,
  }
}
\usepackage{amsfonts}
\usepackage{nicefrac}
\usepackage{microtype}
\usepackage{xcolor}

\input{macro}

\crefname{definition}{Definition}{Definitions}
\Crefname{definition}{Definition}{Definitions}
\crefname{assumption}{Assumption}{Assumptions}
\Crefname{assumption}{Assumption}{Assumptions}
\crefname{lemma}{Lemma}{Lemmas}
\Crefname{lemma}{Lemma}{Lemmas}
\crefname{proposition}{Proposition}{Propositions}
\Crefname{proposition}{Proposition}{Propositions}
\crefname{corollary}{Corollary}{Corollaries}
\Crefname{corollary}{Corollary}{Corollaries}
\crefname{remark}{Remark}{Remarks}
\Crefname{remark}{Remark}{Remarks}

\title{\textsc{SMCEvolve}: Principled Scientific Discovery via Sequential Monte Carlo Evolution}

\begin{document}

\author{%
  Jiachen Jiang\thanks{Equal contribution. \quad $^{\dagger}$Corresponding author.}\ \ , Huminhao Zhu$^{*}$, Zhihui Zhu$^{\dagger}$ \\
  Department of Computer Science and Engineering \\
  The Ohio State University \\
  \texttt{\{jiang.2880, zhu.4228, zhu.3440\}@osu.edu}
}

\maketitle

\begin{abstract}
LLM-driven program evolution has emerged as a powerful tool for automated scientific discovery, yet existing frameworks offer no principled guide for designing their individual components and provide no guarantee that the search converges. We introduce \textsc{SMCEvolve}, which recasts program search as sampling from a reward-tilted target distribution and approximates it with a Sequential Monte Carlo (SMC) sampler. From this view, three core mechanisms emerge as principled components: adaptive parent resampling, mixture of mutation with acceptance, and automatic convergence control. We further provide a finite-sample complexity analysis that bounds the LLM-call budget required to reach a target approximation error. Across math, algorithm efficiency, symbolic regression, and end-to-end ML research benchmarks, \textsc{SMCEvolve} surpasses state-of-the-art evolving systems while using fewer LLM calls under self-determined termination. The code is available at \url{https://github.com/kongwanbianjinyu/SMCEvolve}.
\end{abstract}

\begin{figure}[h]
  \centering
  \includegraphics[width=\linewidth]{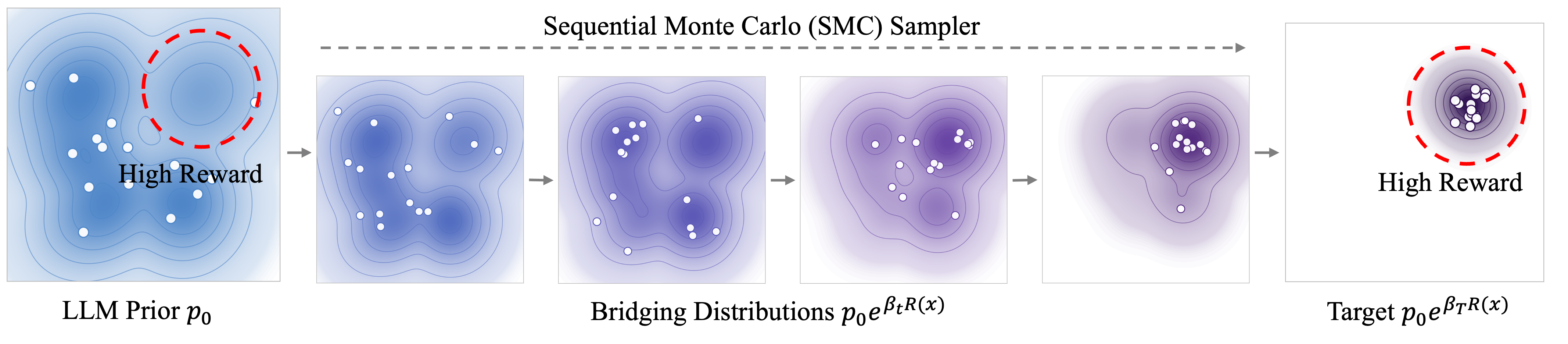}
  \caption{\textbf{LLM-driven program evolution targeting the reward-tilted distribution via Sequential Monte Carlo (SMC).} The contour depicts the distribution over the program space; each white dot is a particle (program) in the evolving population, and the dashed circle marks the high-reward region. Starting from particles drawn from the LLM prior $p_0$ (\emph{left}), where the high-reward region is rarely covered, the population is driven by the LLM toward new programs that form a sequence of bridging distributions $p_0 \, e^{\beta_t R(x)}$ with increasing reward intensity $\beta_t$, and ultimately concentrates on the target $p_0 \, e^{\beta_T R(x)}$ (\emph{right}). This yields a principled evolutionary framework \textsc{SMCEvolve} with convergence guarantees, in contrast to ad-hoc pipelines.}
  \label{fig:teaser}
\end{figure}

\input{sections/1_introduction}

\input{sections/2_method}
\input{sections/3_convergence}
\input{sections/4_experiments}

\input{sections/5_ablation}

\section{Conclusion}
\label{sec:conclusion}

{In this paper, we propose \textsc{SMCEvolve}, recasting LLM-driven program evolution as Sequential Monte Carlo sampling from a reward-tilted target distribution. Three coupled mechanisms emerge from this single principle, yielding the first finite-sample bound for evolutionary program search and state-of-the-art performance across domains. Future directions include tightening the bound and extending to multi-objective or reward-free settings.}
\section*{Acknowledgements}

We acknowledge support from NSF grants IIS-2312840 and ECCS-2409701.

\bibliography{reference}
\bibliographystyle{ieeetr}

\input{sections/7_appendix}

\end{document}

%% file: macro.tex
\usepackage{amsmath,amssymb,amsbsy,amsfonts}
\usepackage{amsthm}
\usepackage{aliascnt}
\usepackage{cleveref}
\usepackage{graphicx}
\usepackage{xcolor}
\usepackage{algorithm}
\usepackage{algorithmic}
\usepackage{subcaption}
\usepackage{wrapfig}

\usepackage{enumitem}

\newtheorem{theorem}{Theorem}[section]

\newaliascnt{proposition}{theorem}

\aliascntresetthe{proposition}

\newaliascnt{lemma}{theorem}
\newtheorem{lemma}[lemma]{Lemma}
\aliascntresetthe{lemma}

\newaliascnt{corollary}{theorem}

\aliascntresetthe{corollary}

\newaliascnt{definition}{theorem}
\newtheorem{definition}[definition]{Definition}
\aliascntresetthe{definition}

\newaliascnt{assumption}{theorem}

\aliascntresetthe{assumption}

\newaliascnt{remark}{theorem}

\aliascntresetthe{remark}

\newcommand{\R}{\mathbb{R}}
\newcommand{\E}{\operatorname{\mathbb{E}}}

\newcommand{\pzero}{p_0}
\newcommand{\pstar}{p^{*}}

\newcommand{\calB}{\mathcal{B}}
\newcommand{\calC}{\mathcal{C}}

\newcommand{\calX}{\mathcal{X}}

\usepackage{multirow}
\usepackage[table]{xcolor}
\definecolor{lblueclr}{HTML}{EFF7FE}

\providecommand{\lblue}{}
\makeatletter
\@ifundefined{lblueclr}{%
  \definecolor{lblueclr}{HTML}{EFF7FE}%
  \renewcommand{\lblue}{\cellcolor{lblueclr}}%
}{}
\makeatother

%% file: sections/1_introduction.tex
\section{Introduction}
\label{sec:intro}

LLM-driven agents have enabled automated scientific discovery at scale: given the problem description and evaluator, the agent can autonomously search for programs that optimize a target objective, iteratively refining candidates based on evaluation feedback~\citep{ye2024reevo, liu2024evolution, romera2024mathematical,novikov2025alphaevolve, lange2025shinkaevolve}. This paradigm has rapidly expanded across diverse domains, including mathematical discovery~\citep{georgiev2025mathematical},  symbolic regression~\citep{shojaee2024llm}, materials design~\citep{abhyankarllema}, machine learning~\citep{chan2024mle, imajuku2025ale} and end-to-end research automation~\cite{yamada2025ai}.

Despite these successes, current agent design remains largely
\emph{empirical} and the area lacks a unifying \emph{theoretical}
framework that guides the design. A typical pipeline, exemplified by
AlphaEvolve~\cite{novikov2025alphaevolve}, ShinkaEvolve~\cite{lange2025shinkaevolve}, and others~\cite{ye2024reevo, liu2024evolution, jiang2026deltaevolve, cemri2026adaevolve, yan2026pacevolve, liu2026evox}, maintains a population of
candidate programs, prompts an LLM to mutate a selected parent based
on context, and adds the evaluated offspring back into the population.
The absence of theoretical guidance leaves several limitations. First,
there is no principled guide for designing each component within the
pipeline: components such as parent selection, context construction,
and population management are hand-crafted in isolation and validated
through extensive ablations rather than derived from a coherent
principle that connects them. Second, the convergence of the pipeline
is not guaranteed, leading to substantial randomness and high variance
across runs;
moreover, existing systems cannot distinguish whether the population
has converged to a high-reward region or merely stalled at a local
mode
, forcing the number of iterations to be fixed in advance—either too small to reach a good solution, or too large and wasting compute. These observations motivate the central question of this work:

\begin{center}
\emph{Can we ground LLM-driven program evolution in a rigorous probabilistic framework that unifies its components and provides convergence guarantees?}
\end{center}

We address this question from a simple observation: searching for high-reward programs with an LLM can be cast as \emph{sampling from a reward-tilted target distribution} $p^*(x \mid q) \propto p_0(x \mid q)\, e^{\beta R(x)}$, which reweights the LLM prior $p_0$ over programs $x$ for a given problem $q$ toward those with high evaluation reward $R$. Since directly sampling from $p^*$ is intractable, we turn to \emph{Sequential Monte Carlo (SMC) samplers} \citep{del2006sequential, dai2022invitation} which progressively transport particles from the prior $p_0$ toward the target $p^*$ through a sequence of bridging distributions with increasing reward weight. Treating each candidate program as a particle, we derive
\textsc{SMCEvolve}, a complete evolutionary framework whose three
core mechanisms emerge directly from SMC theory:

\begin{itemize}[leftmargin=1.5em,itemsep=2pt,topsep=2pt]

\item \textbf{Adaptive parent resampling.}
Parents are resampled by importance weights derived from the current
bridging distribution, where the influence of reward on selection
strengthens as the search progresses—from near-uniform sampling that
encourages exploration in early iterations, to sharply reward-focused
sampling that drives exploitation in later iterations. This stands in
contrast to existing systems, which rely on a fixed selection rule
throughout the search.

\item \textbf{Mixture of mutation with acceptance.} 
To ensure convergence, the LLM generates multiple proposals at each step,
each accepted or rejected with a probability related to the current bridging
distribution. To accelerate convergence, we further design a mixture
of LLM-driven kernels along two axes—edit scope (local diff vs. full
rewrite) and information source (single-parent vs. cross-program
inspiration)—and adaptively select among them via Thompson sampling~\citep{russo2018tutorial},
which favors kernels that have recently produced high-reward
proposals. In contrast, existing systems accept the proposal unconditionally and rely on either a single mutation strategy or a fixed-weight combination of strategies.

\item \textbf{Automatic convergence control.} The number of iterations is determined adaptively from the state of
the population rather than fixed in advance: the search proceeds in
small steps when the population is still spread across diverse
candidates and in larger steps once it concentrates on high-reward
programs, and terminates automatically once it reaches the target
distribution.
In contrast, existing systems run for a fixed number of
iterations specified in advance, without a signal to detect
convergence or stalling.
\end{itemize}

Together, these three mechanisms instantiate a complete SMC sampler
for LLM-driven program search, and the resulting framework admits
formal convergence analysis with finite-sample complexity guarantee. Our main contributions are as follows:

\begin{itemize}[leftmargin=1.5em,itemsep=2pt,topsep=2pt]

\item \textbf{A principled framework for evolutionary program search.}
We introduce \textsc{SMCEvolve}, the first evolving agent
whose components—parent resampling, mutation, and convergence
control—are derived from SMC theory rather than
hand-crafted in isolation.

\item \textbf{Finite-sample convergence analysis.} 
We provide the first convergence guarantee for program evolution, showing
that achieving a target approximation error requires a total
computational cost bounded explicitly in terms of particle budget,
mutation mixing, and reward intensity.

\item \textbf{Empirical validation across diverse domains.}  We
evaluate \textsc{SMCEvolve} on multiple domains spanning math~\citep{novikov2025alphaevolve}, algorithm efficiency~\citep{press2025}, symbolic regression~\citep{shojaee2025llm}, and LLM pretraining~\citep{karpathy_autoresearch_2026}, where it surpasses state-of-the-art evolving
systems on the majority of tasks while consuming fewer
LLM calls under automatic convergence control.

\end{itemize}

%% file: sections/2_method.tex
\begin{figure}[t]
  \centering
\includegraphics[width=\linewidth]{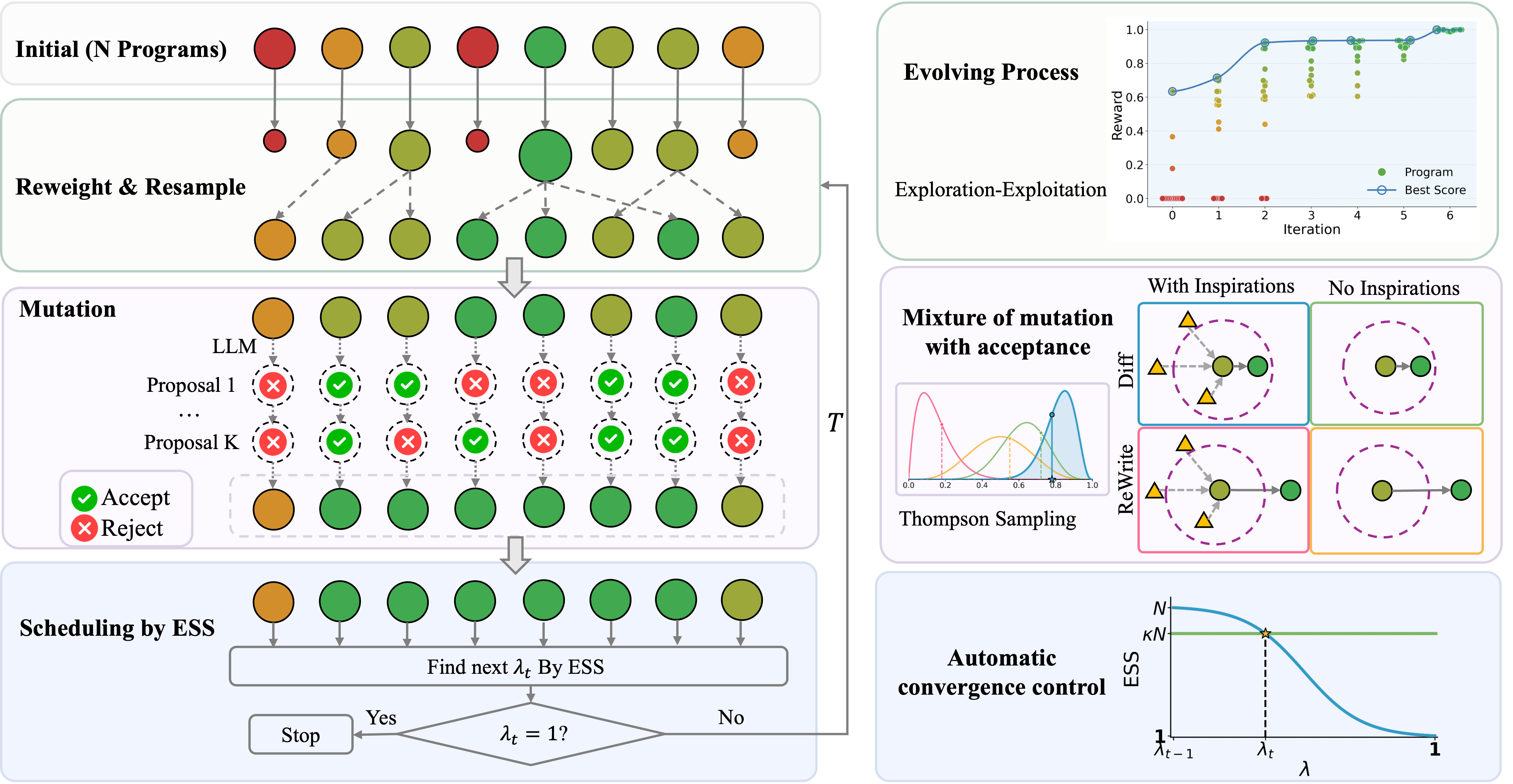}
  \caption{\textbf{Overview of \textsc{SMCEvolve}.}
  \emph{Left:} the main loop. Starting from $N$ initial programs, each of $T$ iterations \emph{(i)} reweights and resamples particles, \emph{(ii)} {mutates each parent through a $K$-step MH chain of LLM proposals and accept/reject updates}, and \emph{(iii)} schedules the next $\lambda_t$ via ESS, terminating when $\lambda_t = 1$.
  \emph{Right, top to bottom:} the evolving process on circle packing (best-so-far reward across iterations); the mixture of four mutation kernels selected by Thompson sampling; and the automatic convergence control finding $\lambda_t$ at the intersection of $\mathrm{ESS}(\lambda)$ with the threshold $\kappa N$.}
  \vspace{-.2in}
  \label{fig:main}
\end{figure}
\section{Method: Program Evolution as Sequential Monte Carlo} 
\label{sec:method}

We formulate program search as sampling from a reward-tilted target distribution that balances reward maximization against the LLM prior (\Cref{sec:ideal}), and approximate this intractable target with a Sequential Monte Carlo sampler whose components -- intermediate distributions, importance weights, mutation kernel, and adaptive scheduling -- are derived from SMC theory (\Cref{sec:smc}).

\paragraph{Notation.}
Throughout, let $\calX$ denote the countable program space for a task $q$.
All distributions over $\calX$ are probability mass functions, and integrals
$\int_{\calX} f(x)\,dx$ denote sums over programs. Given an LLM, let $p_{\mathrm{LLM}}(\cdot \mid \calC)$ denote the distribution induced by sampling from the LLM under context $\calC$.
We define $p_0(\cdot\mid q) = p_{\mathrm{LLM}}(\cdot \mid q)$ as
the LLM prior for task $q$, and $Q_t(\cdot\mid x,\calC_t) = p_{\mathrm{LLM}}(\cdot \mid x,\calC_t)$ as the stage-conditional LLM proposal distribution used for mutation, conditioned on the current program $x$ and context $\calC_t$. Let $R:\calX\to\R$ be the reward function, assumed bounded on
$\mathrm{supp}(p_0)$. Write $R_-=\inf_{\mathrm{supp}(p_0)}R$,
$R_+=\sup_{\mathrm{supp}(p_0)}R$, and $\Delta_R=R_+-R_-<\infty$.

\subsection{Reward-Tilted Target Distribution}
\label{sec:ideal}

Across various domains including mathematical discovery~\citep{georgiev2025mathematical},  symbolic regression~\citep{shojaee2024llm}, materials design~\citep{abhyankarllema}, scientific discovery shares 
a common structure: a search problem in program space $\calX$ guided by an evaluator, where the goal 
is to find the program that maximize the reward $R(x)$. However, the program 
space is discrete and large, making direct maximization intractable. 
We therefore leverage an LLM prior $p_0$ to guide the search toward plausible programs. To balance reward maximization with adherence to the prior, we formulate LLM-driven program search as the following KL-regularized variational optimization problem:
\begin{equation}\label{eq:variational}
\max_{p} \; \E_{x \sim p}\!\big[R(x)\big] - \frac{1}{\beta}\, D_{\mathrm{KL}}\!\big(p \,\|\, p_0\big),
\end{equation}
where the maximization is over all probability mass functions on $\calX$ that are absolutely continuous with respect to $p_0$. The following result characterize the optimal solution.
\begin{lemma}[Reward-tilted target distribution]\label{def:ideal}
The unique maximizer of~\eqref{eq:variational} is
\begin{equation}\label{eq:pstar}
\pstar(x \mid q) \;\triangleq\; \frac{1}{Z(q)}\, \pzero(x \mid q)\, e^{\beta R(x)},
\end{equation}
where $Z(q) = \int \pzero(x' \mid q)\, e^{\beta R(x')}\, dx'$ is the partition function.
\end{lemma}
The proof follows from standard calculus of variations: setting the functional derivative of the Lagrangian (with a normalization constraint) to zero yields~\eqref{eq:pstar} directly, see Appendix~\ref{app:derivation} for details. The bounded-oscillation assumption ensures that $e^{\beta R_-} \le Z(q) \le e^{\beta R_+} < \infty$, so~\eqref{eq:pstar} is a well-defined probability mass function on $\calX$ that is absolutely continuous with respect to $p_0$.

The parameter $\beta$ interpolates between the LLM prior ($\beta \to 0$) and pure reward maximization ($\beta \to \infty$). Sampling from $p^*$ is intractable, as the normalization term $Z(q)$ requires summing over all programs in an extreme large program space.

\subsection{Sequential Monte Carlo Samplers}
\label{sec:smc}

To sample from the intractable target $p^*(\cdot \mid q)$, we turn to Sequential Monte Carlo (SMC) method~\citep{del2006sequential}. It bridges $p_0$ to $p^*$ through a sequence of \emph{bridging distributions} $p_0, p_1, \ldots, p_T = p^*$, with $p_t$ written as $p_t(x) = \gamma_t(x) / Z_t$ for an unnormalized density $\gamma_t$ and normalizing constant $Z_t$.
The sampler maintains $N$ particles $\{x_{t-1}^{(n)}\}_{n=1}^N$ approximating $p_{t-1}$, and advances them as samples from $p_t$ by repeating three operations: \emph{reweighting} each particle by an importance weight, \emph{resampling} them to discard low-weight particles, and \emph{mutating} each particle to inject diversity.
The reweighting step assigns the importance weight as~\citep{dai2022invitation, del2006sequential} 
\begin{equation}\label{eq:smc-weight-general}
w_t(x_{t-1}, x_t) \;=\; \frac{\gamma_t(x_t)\, L_{t-1}(x_t, x_{t-1})}{\gamma_{t-1}(x_{t-1})\, M_t^K(x_{t-1}, x_t)},
\end{equation}
where $(x_{t-1}, x_t) = (x_{t-1}^{(n)}, x_t^{(n)})$ for each particle $n$, $M_t^K(x_{t-1}, \cdot)$ is the \emph{stage forward kernel}---the $K$-fold composition of a one-step kernel $M_t$ used to mutate each particle from stage $t-1$ to stage $t$---and $L_{t-1}(x_t, \cdot)$ is an auxiliary \emph{backward kernel} introduced by the SMC framework to make~\eqref{eq:smc-weight-general} tractable. Crucially, the effectiveness of SMC hinges on how these components are \emph{jointly designed for the target problem}: the bridging distributions $\{p_t\}$ determine the annealing path, $M_t$ controls exploration and proposal quality, and the backward kernel $L_{t-1}$ stabilizes importance weighting.

\paragraph{Kernel properties required by SMC.}
A forward kernel $M_t$ acts on distributions in two cases: starting from $p_t$,
applying $M_t$ should keep the distribution at $p_t$
(\emph{$p_t$-invariance}); starting away from $p_t$, iterating $M_t$ should
drive the distribution toward $p_t$ (\emph{ergodicity}).
Invariance ensures that the move step,
applied to a particle already distributed as $p_t$, does not drift the
distribution away from $p_t$. {But invariance alone is not enough for finite-budget search. 
The identity kernel, for instance, is perfectly $p_t$-invariant but simply copies particles and provides no exploration after resampling. 
Uniform ergodicity rules out such ineffective kernels by giving a quantitative total-variation bound after $K$ mutation steps; smaller $\rho$ corresponds to faster mixing.}

\begin{definition}[$p_t$-invariance]
\label{def:invariance}
A Markov kernel $M_t$ on $\calX$ is \emph{$p_t$-invariant} if
applying $M_t$ to a state distributed as $p_t$ leaves the distribution
unchanged:
\begin{equation}
\label{eq:invariance}
\int_{\calX} p_t(x)\, M_t(x, y)\, dx \;=\; p_t(y), \qquad \forall\, y \in \calX.
\end{equation}
\end{definition}

\begin{definition}[Uniform ergodicity; {\citealp[Ch.~16]{meyn2012markov}}]
\label{def:uniform_ergodicity}
A Markov kernel $M_t$ with invariant distribution $p_t$ is \emph{uniformly
ergodic} if there exist constants $C \geq 1$ and $\rho \in (0, 1)$ such that,
for every $x \in \calX$ and every $K \geq 1$,
\begin{equation}
\label{eq:uniform_ergodicity_def}
\big\| M_t^K(x, \cdot) - p_t \big\|_{\mathrm{TV}} \;\leq\; C \cdot \rho^K,
\end{equation}
where $\|\mu - \nu\|_{\mathrm{TV}} = \sup_{A \subseteq \calX} |\mu(A) - \nu(A)|$ is
the total variation distance.
\end{definition}

\subsection{\textsc{SMCEvolve}: LLM-driven program evolution via SMC}

{Building on this perspective, we instantiate SMC for program evolution by designing problem-specific components tailored to LLM-based search. Our method, \textsc{SMCEvolve}, integrates (i) \emph{adaptive parent resampling} aligned with reward shaping, (ii) a \emph{mixture of mutation kernels with acceptance correction} to balance exploration and validity, and (iii) \emph{automatic convergence control} via adaptive annealing. These designs are summarized in \Cref{fig:main}, with a case study on circle packing shown in \Cref{fig:case_study}.}

\begin{figure}[t]
    \centering
    \begin{subfigure}[b]{0.25\textwidth}
        \centering
        \includegraphics[width=\linewidth]{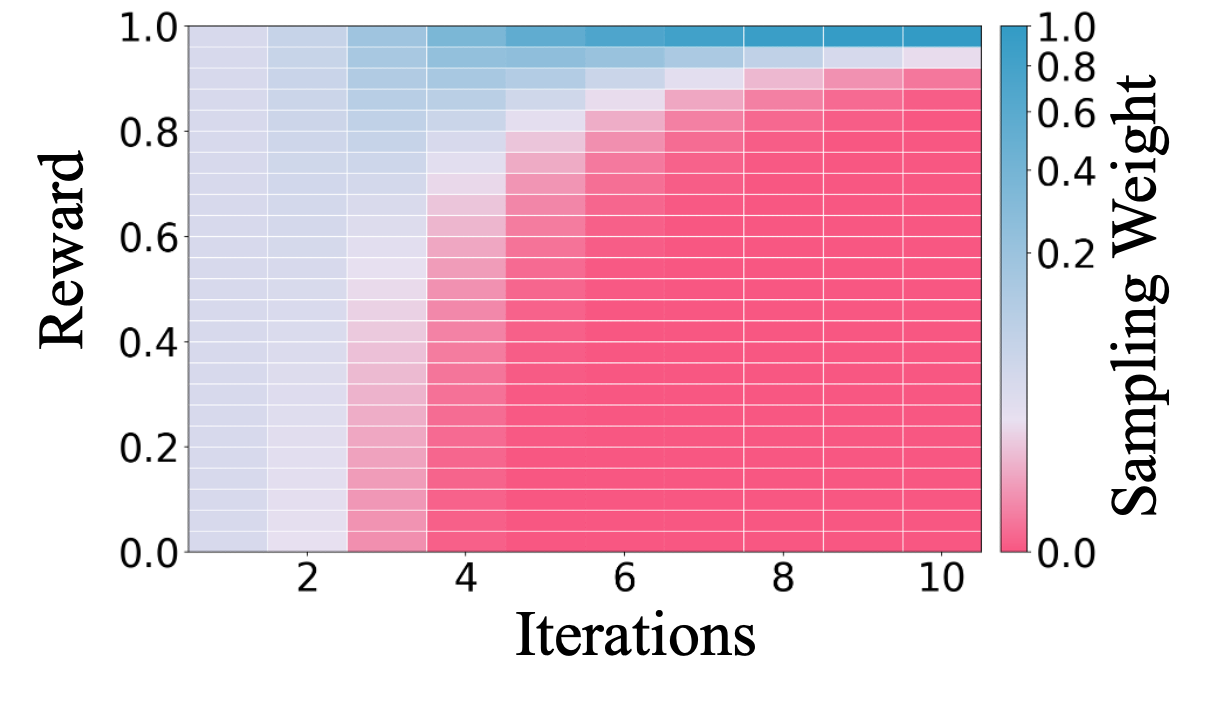}
        \caption{Parent resampling}
        \label{fig:parent_score}
    \end{subfigure}
    \hfill
    \begin{subfigure}[b]{0.32\textwidth}
        \centering
        \includegraphics[width=\linewidth]{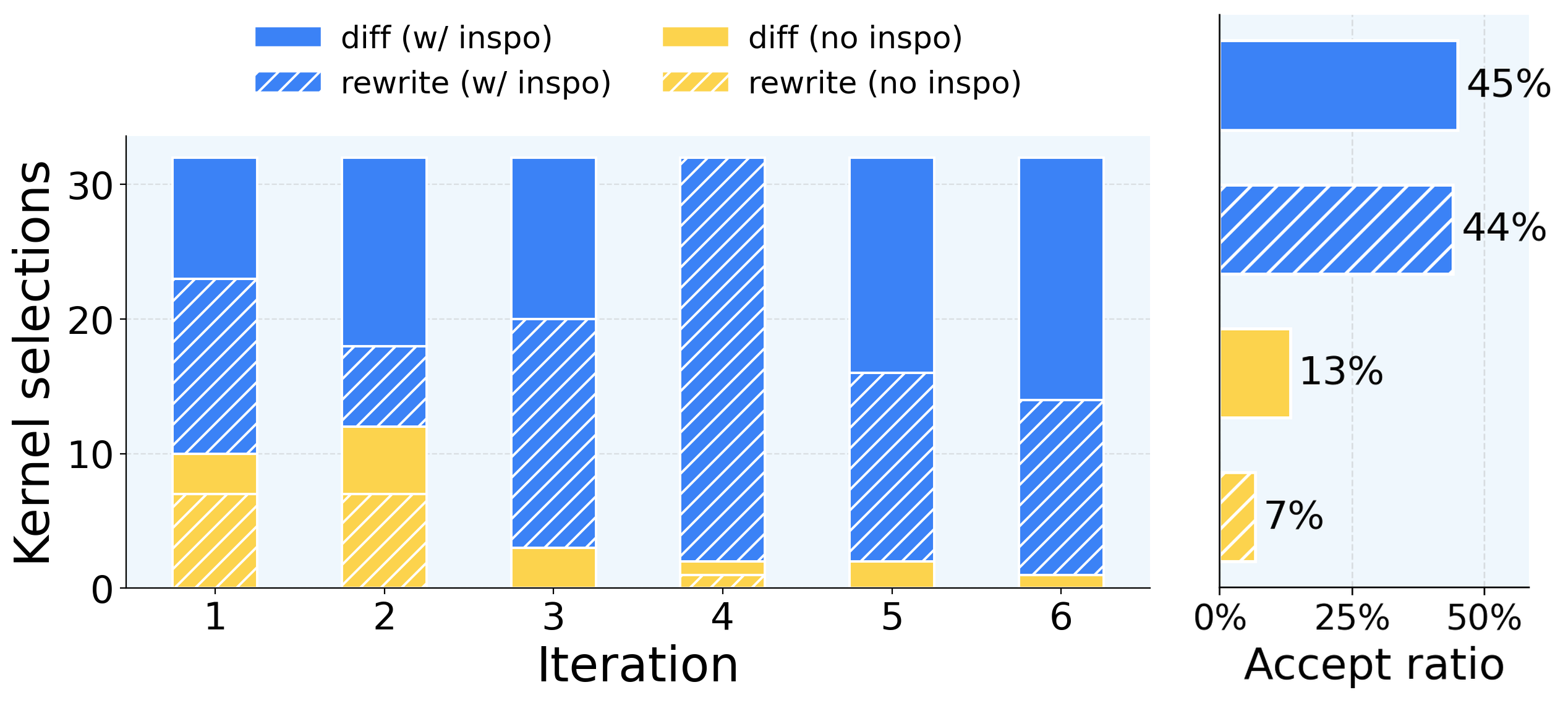}
        \caption{Kernel selection \& acceptance}
        \label{fig:kernel}
    \end{subfigure}
    \hfill
    \begin{subfigure}[b]{0.4\textwidth}
        \centering
        \includegraphics[width=\linewidth]{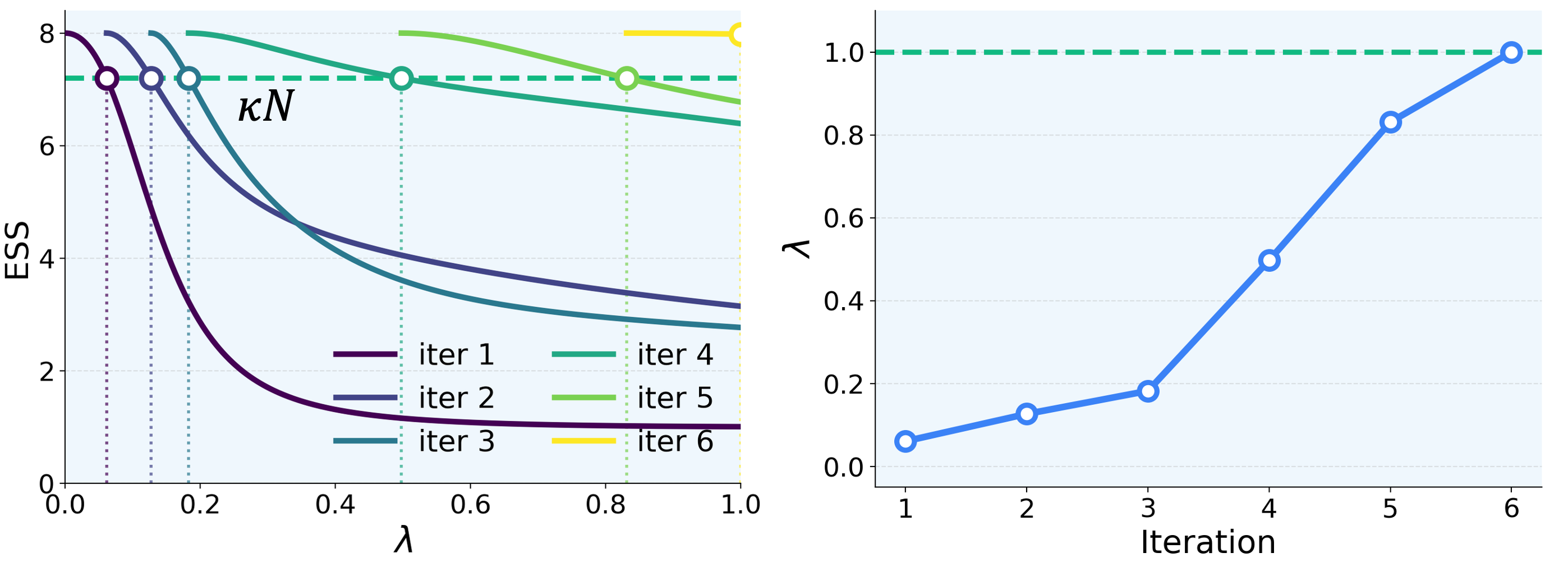}
        \caption{$\mathrm{ESS}$ curves \& $\lambda_t$ schedule}
        \label{fig:lambda_ess}
    \end{subfigure}

    \caption{\textbf{Case study of \textsc{SMCEvolve} on circle packing}, illustrating the three mechanisms.
    \textbf{(a)} Resampling probability of each particle (ranked by reward) across iterations.
    \textbf{(b)} Per-iteration selection counts and cumulative MH acceptance rates of the four mutation kernels $M^{(1)},\dots,M^{(4)}$.
    \textbf{(c)} The ESS curves $\mathrm{ESS}(\lambda)$ at each iteration (left) and the resulting schedule $\lambda_t$ (right); each $\lambda_t$ is set where $\mathrm{ESS}(\lambda)$ crosses the threshold $\kappa N$ (green dashed).}
    \label{fig:case_study}
     \vspace{-.2in}
\end{figure}

\paragraph{Adaptive parent resampling.}
We construct the bridging distributions $\{p_t\}_{t=0}^T$ as a geometric annealing path between the LLM prior $p_0(\cdot \mid q)$ and the target $p^*(\cdot \mid q)$.
At iteration $t$, we set the unnormalized density to, 
\begin{equation}\label{eq:bridging}
\gamma_t(x) \;=\; p_0(x \mid q)\, \exp\!\big(\beta_t R(x)\big), \qquad \beta_t \;=\; \lambda_t\, \beta,
\end{equation}
where the temperature schedule $0 = \lambda_0 < \lambda_1 < \cdots < \lambda_T = 1$ interpolates the reward intensity from zero to its target value $\beta$.
The endpoints recover $p_0(\cdot \mid q)$ at $t=0$ and $p^*(\cdot \mid q)$ at $t=T$ from~\eqref{eq:pstar}.

Next, we specialize the backward kernel $L_{t-1}$ as the time reversal~\citep{dai2022invitation} of $M_t^K$ under $p_t$,
\begin{equation}\label{eq:backward}
L_{t-1}(x_t, x_{t-1}) \;=\; \frac{p_t(x_{t-1})\, M_t^K(x_{t-1}, x_t)}{p_t(x_t)}.
\end{equation}
Substituting~\eqref{eq:backward} into~\eqref{eq:smc-weight-general}, the forward and backward kernels cancel, leaving a weight that is independent of $x_t$ and depends only on the parent's reward (full derivation in Appendix~\ref{app:weights}):
\begin{equation}\label{eq:weight}
w_t(x_{t-1},x_t) \;=\; \frac{\gamma_t(x_{t-1})}{\gamma_{t-1}(x_{t-1})} \;=\; \exp\!\big((\beta_t - \beta_{t-1})\, R(x_{t-1})\big).
\end{equation}
Normalizing across the $N$ particles, the resampling probability of particle $n$ is the temperature-based softmax over particle rewards,
\begin{equation}\label{eq:softmax}
W_t^{(n)} \;=\; \frac{\exp\!\big((\beta_t - \beta_{t-1})\, R(x^{(n)}_{t-1})\big)}{\sum_{m=1}^{N} \exp\!\big((\beta_t - \beta_{t-1})\, R(x^{(m)}_{t-1})\big)},
\end{equation}
where $\beta_t - \beta_{t-1}>0$ acts as inverse temperature: small increments flatten the softmax toward uniform sampling, while large increments sharpen it onto top-reward particles.

\textit{Case study.} \Cref{fig:parent_score} illustrates this transition on circle packing. In early iterations the increment $\beta_t - \beta_{t-1}$ is small, so reward exerts little influence on the resampling weight and sampling probabilities are nearly uniform across particles; in later iterations the increment grows, and the weight concentrates resampling on high-reward parents.

\paragraph{Design of the forward kernel.}
Leveraging the in-context learning capability of LLMs, we design the one-step forward/mutation kernel $M_t$ around an LLM-based proposal
$x'\sim Q_t(\cdot\mid x,\mathcal C_t)$. However, the LLM proposal kernel $Q_t$ alone may not satisfy the two properties required for SMC convergence: $p_t$-invariance~(\Cref{def:invariance}) and uniform ergodicity~(\Cref{def:uniform_ergodicity}). To address this challenge, we augment the proposal mechanism with a mixture of mutation kernels and a Metropolis–Hastings (MH) accept/reject correction applied to multiple candidate samples. Specifically, each parent particle is propagated by repeatedly applying a one-step mutation kernel for $K$ iterations, yielding the resulting $K$-step mutation block $M_t^K$.

{\textbf{MH accept/reject for $p_t$-invariance.} Invariance ensures that the mutation step does not drift the particle population away from the current bridge $p_t$.
We achieve this via the MH construction (proof in Appendix~\ref{app:MH}): each candidate $x'$ proposed by the LLM is accepted with probability $\alpha_t(x, x')$, and otherwise the parent $x$ is kept.
Iterating this proposal-and-acceptance step $K$ times per particle yields a $K$-step MH chain, whose final state is taken as the mutated particle.} {The full MH ratio requires evaluating the LLM density in both forward and reverse directions, which is inaccessible through black-box APIs. We therefore adopt the reward-only criterion as an approximation:
\begin{equation}\label{eq:mh_simple}
\alpha_t(x, x') = \min\!\big\{1,\; e^{\beta_t(R(x') - R(x))}\big\}.
\end{equation}
The deviation is discussed in Appendix~\ref{app:kernel}. Intuitively, this rule always accepts reward-improving proposals and accepts reward-decreasing ones with probability that decays exponentially in the reward gap. {Under the ideal MH kernel, larger $K$ gives the chain more time to mix toward $p_t$ at the cost of more LLM calls.}
}

{\textbf{LLM proposal with context for ergodicity.}
Ergodicity governs how fast newly generated programs spread across the program space, and we accelerate it through the design of the LLM context $\mathcal{C}_t$.
Following the mixture-of-kernels construction~\citep[Sec.~3.2.2]{del2006sequential}, {we instantiate $Q_t$ as a mixture of four LLM proposal modes; composed with MH accept/reject, these modes induce the four one-step mutation kernels shown below.}}
\begin{center}
\small
\begin{tabular}{l|cc}
\toprule
 & \textbf{With inspiration} & \textbf{No inspiration} \\
\midrule
\textbf{Diff}    & $M^{(1)}$: edit parent using peer programs & $M^{(2)}$: edit parent alone \\
\textbf{ReWrite} & $M^{(3)}$: rewrite parent from scratch using peer programs  & $M^{(4)}$: rewrite parent from scratch \\
\bottomrule
\end{tabular}
\end{center}

Diff kernels produce small \texttt{SEARCH/REPLACE} edits, while rewrite kernels
generate the entire program from scratch. Inspiration programs include the top-$k$ programs by reward and $m$ maximally diverse ones, selected by picking programs whose text embeddings are farthest from the parent across the full evaluation history. Crucially, this history 
contains \emph{every} program ever proposed by the LLM, including those 
rejected by the MH step, so that information from rejected proposals is not 
discarded but as inspiration for future mutations.

As the most useful kernel changes over the process of the search, we adaptively select among the four kernels via Thompson sampling~\citep{russo2018tutorial}:
each kernel maintains a Beta posterior over its success rate, and at each step we
sample from every posterior and pick the
kernels that have recently produced high-reward proposals.  Together, these four kernels cover both local
and global moves and both single-parent and cross-program information,
broadening the reachable set from any given particle and accelerating
mixing. 

\textit{Case study.} \Cref{fig:kernel} confirms this behavior on circle packing. The inspiration-based kernels $M^{(1)}$ and $M^{(3)}$ attain substantially higher MH acceptance under~\eqref{eq:mh_simple} than their no-inspiration counterparts $M^{(2)}$ and $M^{(4)}$, and Thompson sampling progressively selects them more often than the no-inspiration ones.

{\textbf{Automatic convergence control.} 
We now describe an adaptive strategy for controlling the evolution iterations. Towards that goal, rather than pre-specifying the inverse temperature sequence $(\lambda_t)_{t=0}^T$, 
we iteratively determine $\lambda_t$ adaptively using the effective sample size 
(ESS)~\citep{del2006sequential} defined as 
\begin{equation}
\label{eq:ess}
\text{ESS}(\lambda) = 
\frac{\bigl(\sum_{n=1}^N u^{(n)}(\lambda)\bigr)^2}
     {\sum_{n=1}^N \bigl(u^{(n)}(\lambda)\bigr)^2}, 
\qquad 
u^{(n)}(\lambda) = e^{(\lambda - \lambda_{t-1}) \beta R(x_{t-1}^{(n)})}.
\end{equation}
Intuitively, $\text{ESS}(\lambda)$ measures how many particles 
effectively contribute after reweighting—close to $N$ when weights are 
balanced, and dropping toward $1$ as they concentrate on a few high-reward 
particles. Since increasing the temperature $\lambda$ sharpens the reward tilt and decreases $\text{ESS}(\lambda)$, we set $\lambda_t$ to be the largest $\lambda$ 
whose reweighting still keeps a target fraction $\kappa \in (0,1)$ of particles 
effective. This can be computed via bisection on
$\text{ESS}(\lambda_t) = \kappa N$ starting from $\lambda_{t-1}$, and
iterate until $\lambda_t = 1$ (set directly when
$\text{ESS}(1) \ge \kappa N$). The total number of iterations $T$ is therefore determined automatically and adapts to the difficulty of the task.
}

To illustrate, \Cref{fig:lambda_ess} shows how the schedule is built on circle packing. At each iteration we compute the curve $\mathrm{ESS}(\lambda)$ as a function of $\lambda$ (left panel) and get the next $\lambda_t$ at its intersection with the threshold $\kappa N$ (green dashed line). The resulting schedule (right panel) grows slowly at first—small steps because the early ESS curves decay sharply—and accelerates as the curves flatten in later iterations, reaching $\lambda_t = 1$ at automatic termination.

\paragraph{The complete algorithm.}
\Cref{alg:smc} assembles the three designs into a single loop:
at each iteration $t$, we (i) reweight and resample particles via the
temperature-based softmax~\eqref{eq:softmax}, (ii) mutate each parent
through a $K$-step MH chain with {the Thompson-sampled mixture of LLM proposal modes}, and (iii) advance $\lambda_t$ by ESS-based
bisection~\eqref{eq:ess}, terminating once $\lambda_t = 1$. The total
cost is $N \cdot K \cdot T$ LLM calls, where $T$ is determined
adaptively rather than fixed in advance. More implementation details,
including island-based parallelism, are in
Appendix~\ref{app:implementation}.

\begin{algorithm}[t]
\caption{\textsc{SMCEvolve}}
\label{alg:smc}
\begin{algorithmic}[1]
\REQUIRE $N$ particles, {$K$ sequential LLM proposals per particle}, ESS threshold $\kappa \in (0,1)$, prior $p_0(\cdot \mid q)$, LLM proposal $Q_t(\cdot \mid x,\calC_t)$, reward $R(\cdot)$, target inverse temperature $\beta$
\ENSURE Particles approximating $p^*(\cdot \mid q)$
\FOR{$n = 1, \ldots, N$}
    \STATE $x_0^{(n)} \sim p_0(\cdot \mid q)$; \; $R^{(n)} \gets R(x_0^{(n)})$ \hfill \# Initialize from LLM prior
\ENDFOR
\STATE $\lambda_0 \gets 0$, $t \gets 0$
\REPEAT
    \STATE $t \gets t + 1$
    \STATE Determine $\lambda_t$ via ESS bisection~\eqref{eq:ess}; \; $\Delta\beta_t \gets (\lambda_t - \lambda_{t-1})\beta$ \hfill \# Adaptive temperature
    \STATE Compute weights $W_t^{(n)}$ via~\eqref{eq:softmax}; \; resample ancestor indices $\{A_t^{(n)}\}_{n=1}^{N}$ \hfill \# Parent resampling
    \FOR{$n = 1, \ldots, N$}
        \STATE $z \gets x_{t-1}^{(A_t^{(n)})}$ \hfill \# Resampled parent (current state of MH chain)
        \FOR{$k = 1, \ldots, K$}
            \STATE $x' \sim Q_t(\cdot \mid z, \calC_t)$ \hfill \# LLM proposes from current state
            \STATE With probability $\alpha_t(z, x')$ via~\eqref{eq:mh_simple}: $z \gets x'$ \hfill \# MH accept/reject
        \ENDFOR
        \STATE $x_t^{(n)} \gets z$; \; $R^{(n)} \gets R(x_t^{(n)})$
    \ENDFOR
\UNTIL{$\lambda_t = 1$} \hfill \# Automatic termination
\end{algorithmic}
\end{algorithm}

\paragraph{A unified exploration--exploitation transition.}
The three core designs of \textsc{SMCEvolve} are not independent: they 
are coupled through the single temperature parameter $\beta_t$, which 
governs the exploration--exploitation trade-off across all of them. 
At small $\beta_t$, adaptive parent resampling~\eqref{eq:softmax} draws 
nearly uniformly across particles, MH acceptance~\eqref{eq:mh_simple} 
admits almost any LLM proposal, and ESS-driven 
scheduling~\eqref{eq:ess} takes small steps in $\beta_t$---together 
driving exploration. At large $\beta_t$, parent resampling concentrates 
on high-reward particles, MH acceptance admits only reward-improving 
proposals, and ESS-driven scheduling takes larger steps---together 
driving exploitation. Through this single parameter, the three designs 
are no longer independent choices but coordinated mechanisms of one 
exploration--exploitation schedule that emerges from the SMC framework 
itself.

\paragraph{Existing frameworks as special cases.}
Existing agents such as AlphaEvolve~\citep{novikov2025alphaevolve} and 
ShinkaEvolve~\citep{lange2025shinkaevolve} arise as special cases of our 
framework when each component reduces to a specific choice: adaptive 
parent resampling reduces to a fixed inverse temperature 
$\beta_t \equiv \beta$, mixture mutation with acceptance reduces to a 
single LLM proposal ($K = 1$, $\alpha_t \equiv 1$) with unconditional 
acceptance, and automatic convergence control reduces to a pre-fixed 
iteration count $T$. This explains why existing frameworks work, while \textsc{SMCEvolve} provides a more general and unified principle for designing evolving agents.

%% file: sections/3_convergence.tex
\section{Convergence Analysis}
\label{sec:convergence}

Direct sampling from the reward-tilted target $p^*$ is intractable. The SMC
formulation of \textsc{SMCEvolve} instead constructs an interacting particle
approximation through annealed reweighting, resampling, and mutation. We measure
this approximation in the standard SMC sense: for bounded test functions $f$, the
terminal empirical measure $\eta_T^N = N^{-1}\sum_n \delta_{x_T^{(n)}}$ should
satisfy $\eta_T^N(f) \approx p^*(f)$. The result below gives a finite-sample
LLM-call budget for achieving this approximation under explicit bridge
regularity and mutation-kernel assumptions.

Let $\calB := NTK$ denote the total number of LLM calls, accounting for $N$
particles, $T$ annealing stages, and $K$ applications of the one-step mutation
kernel per particle. The
following theorem states the resulting finite-sample complexity bound.

\begin{theorem}[Finite-sample complexity of \textsc{SMCEvolve}]
\label{thm:complexity}
Suppose the reward has finite positive oscillation $\Delta_R$ on the support of
$p_0(\cdot\mid q)$. Let the schedule $\{\lambda_t\}_{t=0}^T$ be produced by the
ESS bisection rule~\eqref{eq:ess} with threshold $\kappa \in (0, 1)$, and
assume the resulting adjacent bridges satisfy
$\sup_{1 \le t \le T}\|p_t / p_{t-1}\|_{L^2(p_{t-1})}^2 \le \kappa^{-1}$.
Assume also that for each $t$, the one-step mutation kernel $M_t$ is
$p_t$-invariant and uniformly ergodic with rate $\rho_t$, and let
$\rho = \max_{1 \le t \le T} \rho_t \in (0, 1)$.

Then for any fixed bounded test function $f: \calX \to \R$ with
$\|f\|_\infty \le 1$ and any $\epsilon > 0$, there exists a choice of
$(N,T,K)$ such that the empirical measure
$\eta_T^N = \frac{1}{N}\sum_n \delta_{x_T^{(n)}}$ induced by the particles
$\{x_T^{(n)}\}$ satisfies\footnote{The constant $3/4$ is the high-probability guarantee in Theorem~1 of Marion et al.~\citep{marion2023finite}; it can be amplified to $1 - \delta$ at a $\log(1/\delta)$ multiplicative cost by running the sampler $O(\log(1/\delta))$ times and taking the median of the estimates.}
\begin{equation}
\label{eq:concentration}
\mathbb{P}\!\left(\big| \eta_T^N(f) - p^*(f) \big| \le \epsilon \right) \;\ge\; \tfrac{3}{4}.
\end{equation}

The corresponding LLM-call budget $\calB = NTK$ satisfies
\begin{equation}
\label{eq:main-bound}
\calB \;=\; \widetilde{\mathcal{O}}\!\left(
\frac{\epsilon^{-2} \vee \kappa^{-1}}{1 - \rho} \cdot \beta \Delta_R
\right),
\end{equation}
where $\widetilde{\mathcal{O}}$ suppresses polylogarithmic factors.
\end{theorem}

The theorem provides the standard SMC guarantee for the terminal empirical measure. Intuitively, for any fixed bounded statistic $f$, the final particle population yields a high-accuracy estimate of its expectation under the ideal reward-tilted target distribution $p^*$. For example, choosing the normalized reward
$f=(R-R_-)/\Delta_R$ with $\Delta_R>0$, \eqref{eq:concentration} ensures that
the mean reward of the final population is close to the mean reward obtained by
exact sampling from $p^*$. The best-so-far optimization performance is evaluated
empirically in \Cref{sec:experiments}. More generally, since \eqref{eq:concentration} holds for every bounded test function $f$, the final particles $\{x_T^{(n)}\}$ collectively provides a consistent approximation to samples from the target distribution $p^*$.

{
The assumptions in \Cref{thm:complexity} correspond to the two key ingredients underlying finite-sample guarantee: neighboring bridge distributions remain sufficiently close, and the mutation kernel adequately explores the target distribution at each stage. 
In \textsc{SMCEvolve}, the ESS-based bisection rule is designed to adaptively construct
short bridges, while a formal theoretical characterization is left for future work. Similarly, the LLM-based mutation kernel is designed to approximately preserve $p_t$-invariance through an MH-style accept/reject mechanism, while improving exploration through the use of four complementary proposal modes. Under the
full MH ratio, detailed balance and hence $p_t$-invariance holds exactly. In practice, because black-box LLM APIs do not expose the quantities required for the full MH correction, we employ the reward-only approximation described in Appendix~\ref{app:kernel}. The acceptance-rate and kernel-selection diagnostics in \Cref{fig:kernel} suggest that this approximation remains effective in practice and provides a reasonable surrogate for the idealized invariant kernel.
}

\paragraph{Proof sketch.}
{\Cref{thm:complexity} is a specialization of Theorem~1 and Sec.~5.1
of Marion et al.~\citep{marion2023finite} in its mixing-time form. The imported
bound has the form
$\calB = \widetilde{\mathcal{O}}\!\big((\epsilon^{-2} \vee \gamma^2) \cdot \tau \cdot \log\Gamma\big)$,
where $\gamma^2$ measures bridge regularity, $\tau$ measures mutation mixing
time, and $\Gamma = \sup_{x:p_0(x\mid q)>0} p^*(x)/p_0(x)$ measures the length
of the annealing path. We map these three quantities to
\textsc{SMCEvolve} as follows:}
\begin{itemize}[leftmargin=*,topsep=2pt]
\item \emph{(S1: bridge term) $\gamma^2 \le \kappa^{-1}$.} {The theorem assumes that adjacent bridges have controlled $L^2$ ratio;
the ESS bisection rule~\eqref{eq:ess} is the finite-sample mechanism that
targets this condition. For the reward-tilted path, short temperature steps
also satisfy the deterministic certificate
$\|p_t/p_{t-1}\|_\infty \le e^{\Delta\beta_t \Delta_R}$.}
\item \emph{(S2: mutation term) $\tau \le 1/(1-\rho)$ via uniform ergodicity.}
{Uniform ergodicity says that $K$ steps of
the one-step kernel move any starting state toward the current bridge $p_t$ at
rate $\rho^K$. Inverting this decay gives a mixing cost proportional to
$(1-\rho)^{-1}$ up to logarithmic factors.}
\item \emph{(S3: path-length term) $\log\Gamma \le \beta \Delta_R$.}
{Since
$p^*(x)/p_0(x\mid q)=e^{\beta R(x)}/Z(q)$ on the support of $p_0$, bounded
reward oscillation gives
$\Gamma \le e^{\beta R_+}/e^{\beta R_-}=e^{\beta\Delta_R}$ and hence
$\log\Gamma \le \beta\Delta_R$.}
\end{itemize}
{
Full derivations are in
Appendix~\ref{app:imported_complexity}.}

\paragraph{Regime of validity.}
{\Cref{thm:complexity} bounds the product budget $\calB=NTK$. This
product bound should be read together with the usual SMC requirement that each
factor is large enough: enough particles for stable resampling, enough mutation
steps for mixing, and enough bridge stages for controlled reweighting.
Appendix~\ref{app:regime} records these per-factor requirements and the degenerate
cases with $N=1$ and $K=1$.}

%% file: sections/4_experiments.tex
\section{Experiments}
\label{sec:experiments}

\begin{table}[!t]
  \centering
  \caption{Math domain~\citep{novikov2025alphaevolve}: best reward per task, LLM calls in the \textbf{Calls} column.}
  \label{tab:math}
  \small
  \resizebox{\textwidth}{!}{%
  \begin{tabular}{lccc cc}
  \toprule
  \textbf{Task}
    & \textsc{ReEvo} & \textsc{OpenEvolve} & \textsc{ShinkaEvolve}
    & \multicolumn{2}{c}{\textbf{\textsc{SMCEvolve}}} \\
  \cmidrule(lr){5-6}
    & & & & Best & Calls ($\downarrow$) \\
  \midrule
  Circle Packing in Rect. N=21 ($\uparrow$)
    & 0.9085 & 0.9514 & 0.9514 & \lblue \textbf{0.9993} & \lblue 192 \\
  Hexagon Packing N=11 ($\uparrow$)
    & 0.7790 & 0.9398 & 0.9069 & \lblue \textbf{0.9821} & \lblue 176 \\
  Heilbronn Convex N=13 ($\uparrow$)
    & 0.6933 & 0.7178 & 0.7098 & \lblue \textbf{0.9163} & \lblue 176 \\
  Heilbronn Triangle N=11 ($\uparrow$)
    & 0.6862 & 0.8736 & \textbf{0.9046} & \lblue 0.8966 & \lblue 192 \\
  Min-Max-Min Dist. ($n{=}16, d{=}2$) ($\uparrow$)
    & 0.8634 & 0.9915 & 0.9842 & \lblue \textbf{1.0000} & \lblue 160 \\
  Kissing Number ($d{=}11$) ($\uparrow$)
    & 0.1754 & 0.2530 & 0.7386 & \lblue \textbf{0.9005} & \lblue 144 \\
  1st Autocorrelation Inequality ($\uparrow$)
    & 0.9912 & 0.9893 & 0.9971 & \lblue \textbf{0.9974} & \lblue 176 \\
  2nd Autocorrelation Inequality ($\uparrow$)
    & 0.9990 & 1.0079 & 1.0130 & \lblue \textbf{1.0220} & \lblue 128 \\
  3rd Autocorrelation Inequality ($\uparrow$)
    & 0.9636 & 0.0028 & 0.9942 & \lblue \textbf{0.9981} & \lblue 128 \\
  Erd\H{o}s Minimum Overlap ($\uparrow$)
    & 0.9901 & 0.9973 & 0.9990 & \lblue \textbf{0.9998} & \lblue 192 \\
  \bottomrule
  \end{tabular}%
  }
\end{table}

\begin{table}[!t]
  \centering
  \caption{Algorithm efficiency domain~\citep{press2025}: per-task speedup, LLM calls in the \textbf{Calls} column.}
  \label{tab:algotune}
  \small
  \resizebox{\textwidth}{!}{%
  \begin{tabular}{lccc cc}
  \toprule
  \textbf{Task}
    & \textsc{ReEvo} & \textsc{OpenEvolve} & \textsc{ShinkaEvolve}
    & \multicolumn{2}{c}{\textbf{\textsc{SMCEvolve}}} \\
  \cmidrule(lr){5-6}
    & & & & Best & Calls ($\downarrow$) \\
  \midrule
  \texttt{affine\_transform\_2d} ($\uparrow$)
    & 1.0572 & 0.9957 & 1.1032 & \lblue \textbf{4.7160} & \lblue 400 \\
  \texttt{convolve2d\_full\_fill} ($\uparrow$)
    & 0.9963 & 1.0886 & 1.2682 & \lblue \textbf{11.6000} & \lblue 480 \\
  \texttt{eigenvectors\_complex} ($\uparrow$)
    & 1.0688 & 1.0767 & 1.26 & \lblue \textbf{1.8350} & \lblue 288 \\
  \texttt{fft\_cmplx\_scipy\_fftpack} ($\uparrow$)
    & 0.9785 & 1.3590 & 1.1001 & \lblue \textbf{9.2280} & \lblue 464 \\
  \texttt{fft\_convolution} ($\uparrow$)
    & 1.0756 & 1.6226 & 1.9817 & \lblue \textbf{19.9000} & \lblue 320 \\
  \texttt{lu\_factorization} ($\uparrow$)
    & 1.2854 & 1.8120 & 1.4119 & \lblue \textbf{6.1600} & \lblue 368 \\
  \texttt{polynomial\_real} ($\uparrow$)
    & 1.4333 & 1.6908 & \textbf{33.8776} & \lblue 2.4000 & \lblue 336 \\
  \texttt{psd\_cone\_projection} ($\uparrow$)
    & 1.0501 & 0.9753 & 34.5129 & \lblue \textbf{45.8200} & \lblue 416 \\
  \bottomrule
  \end{tabular}%
  }
\end{table}

\begin{table}[!t]
\centering
\caption{Symbolic regression domain~\citep{shojaee2025llm}: best reward per task, LLM calls in the \textbf{Calls} column.}
\label{tab:sr}
\small
\resizebox{\textwidth}{!}{%
\begin{tabular}{llccc cc}
\toprule
\textbf{Split} & \textbf{Task}
  & \textsc{ReEvo} & \textsc{OpenEvolve} & \textsc{ShinkaEvolve}
  & \multicolumn{2}{c}{\textbf{\textsc{SMCEvolve}}} \\
\cmidrule(lr){6-7}
  & & & & & Best & Calls ($\downarrow$) \\
\midrule
\multirow{4}{*}{\texttt{bio\_pop\_growth}}
  & BPG0 & 4.9052 & 5.0918 & 6.4700 & \lblue \textbf{7.1021} & \lblue 384 \\
  & BPG1 & 3.4338 & 2.8472 & 5.8589 & \lblue \textbf{7.1159} & \lblue 336 \\
  & BPG2 & -0.0962 & 5.9800 & 4.2176 & \lblue \textbf{6.4054} & \lblue 128 \\
  & BPG3 & 8.9428 & 8.9602 & 8.8261 & \lblue \textbf{8.9725} & \lblue 304 \\
\midrule
\multirow{4}{*}{\texttt{chem\_react}}
  & CRK0 & 8.9952 & 8.9789 & 8.9846 & \lblue \textbf{8.9985} & \lblue 256 \\
  & CRK1 & 8.9537 & 8.6241 & 8.9830 & \lblue \textbf{8.9972} & \lblue 288 \\
  & CRK2 & 8.9990 & 8.5082 & 8.9906 & \lblue \textbf{9.0000} & \lblue 304 \\
  & CRK3 & 6.8599 & \textbf{8.1226} & 7.5276 & \lblue 7.8772 & \lblue 288 \\
\midrule
\multirow{4}{*}{\texttt{matsci}}
  & MatSci0 & 3.8771 & 4.5670 & 4.2790 & \lblue \textbf{5.5511} & \lblue 304 \\
  & MatSci1 & 0.1518 & 0.1520 & 0.1521 & \lblue \textbf{0.1522} & \lblue 144 \\
  & MatSci2 & 4.4763 & 6.9586 & 6.0224 & \lblue \textbf{8.2500} & \lblue 208 \\
  & MatSci3 & 1.4764 & 3.0537 & 3.1858 & \lblue \textbf{4.4687} & \lblue 304 \\
\midrule
\multirow{4}{*}{\texttt{phys\_osc}}
  & PO0 & 4.1917 & 3.2642 & \textbf{5.6612} & \lblue 4.9567 & \lblue 256 \\
  & PO1 & 3.5738 & 6.5793 & 6.6001 & \lblue \textbf{6.6998} & \lblue 336 \\
  & PO2 & 4.9201 & 7.4520 & \textbf{7.4928} & \lblue 7.4296 & \lblue 368 \\
  & PO3 & 8.9929 & 8.9278 & 8.9486 & \lblue \textbf{8.9985} & \lblue 288 \\
\bottomrule
\end{tabular}%
}
\end{table}

\subsection{Main Results}
\label{sec:exp_main}
We evaluate \textsc{SMCEvolve} on four domains:
\emph{Math} --- combinatorial geometry problems from the
AlphaEvolve benchmark\citep{novikov2025alphaevolve}(Table~\ref{tab:math}); \emph{AlgoTune} --- runtime-optimization tasks~\citep{press2025} spanning numerical computation, signal processing, and linear algebra
(Table~\ref{tab:algotune}); 
\emph{Symbolic Regression} --- problems from
LLM-SRBench~\citep{shojaee2025llm} across physical
oscillators, bacterial population growth, chemical reactions, and materials
science (Table~\ref{tab:sr}); 
and \emph{AutoResearch} --- end-to-end GPT pretraining on
\textsc{TinyStories}~\citep{karpathy_autoresearch_2026}
(Figure~\ref{fig:tinystories-curve}), where reward is
$\max(0,\, 2.0 - \text{val\_bpb})$ on a held-out set; each candidate trains for
60\,s on a single A5000. We compare against \textsc{ReEvo}~\citep{ye2024reevo},
\textsc{OpenEvolve}, and
\textsc{ShinkaEvolve}~\citep{lange2025shinkaevolve}. All methods share an
identical \texttt{gpt-5-mini}+\texttt{gemini-3-flash} ensemble and each cell reports the best score over 3 seeds. Baselines run to
a fixed LLM-call budget (200 for Math and AutoResearch, 400 for Symbolic
Regression, 500 for AlgoTune), while \textsc{SMCEvolve} terminates
automatically via ESS-driven stopping. {\textsc{SMCEvolve} attains the best score on the majority of tasks across all four domains and the highest final reward on AutoResearch. Notably, these gains come with strictly fewer LLM calls---the mean call count under ESS-driven termination is below the fixed baseline budget on every domain (reported in the \textbf{Calls} column).}

\begin{figure}[!t]
    \centering
    \includegraphics[width=0.5\linewidth]{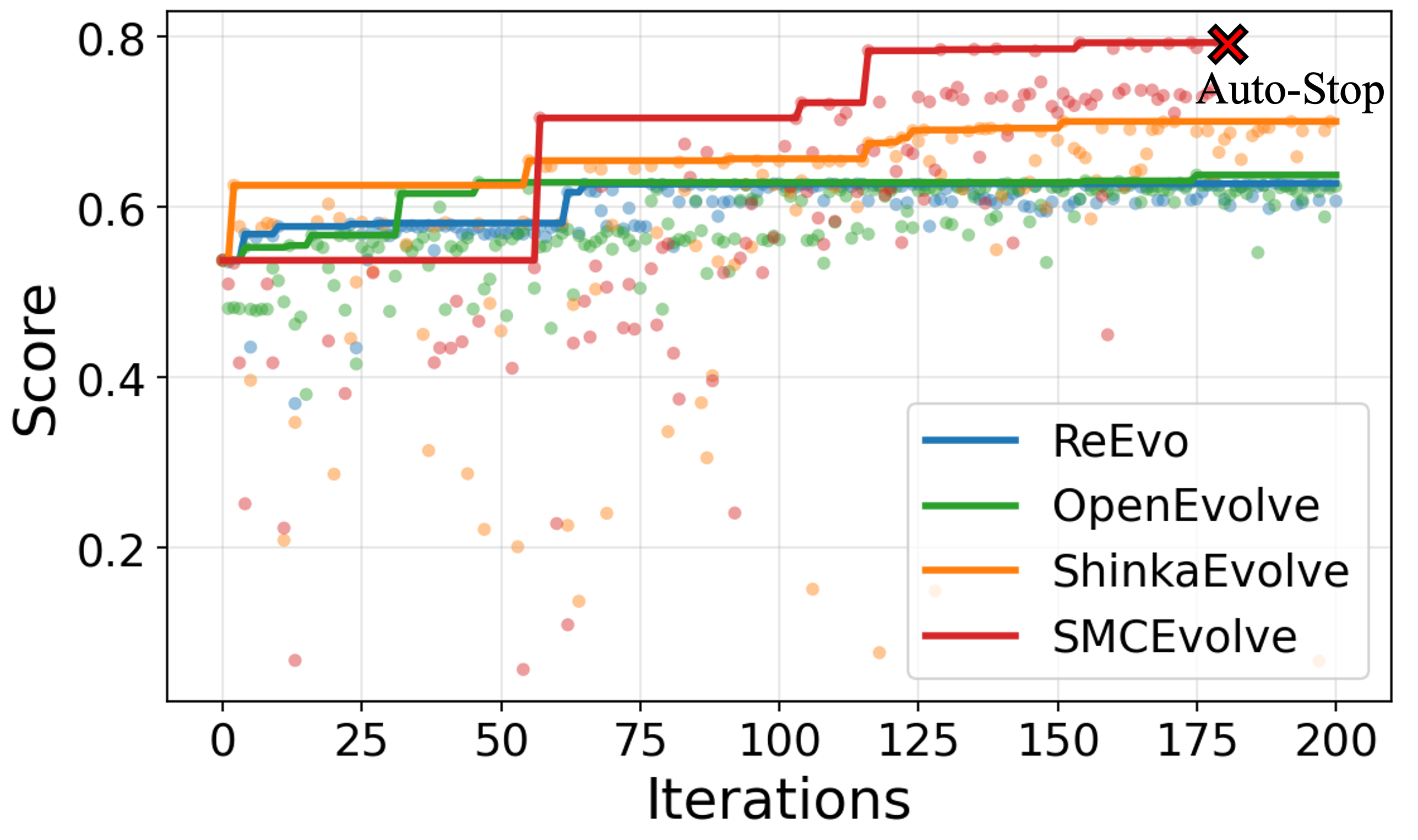}
    \caption{AutoResearch domain~\citep{karpathy_autoresearch_2026}. The SMCEvolve terminated automatically under ESS-control.}
    \label{fig:tinystories-curve}
\end{figure}

%% file: sections/5_ablation.tex
\subsection{Ablation Study}
\label{sec:ablation}

{We ablate the key designs---adaptive parent resampling and the mixture of mutation kernels in \Cref{sec:method}---as well as the hyperparameters $N$ (particle count) and $K$ (MH chain length). Each variant changes a single choice while holding the LLM-call budget fixed (Table~\ref{tab:ablation}). }

\begin{table}[!ht]
\centering
\caption{Ablations on Circle Packing $N{=}21$. Best reward over 3 seeds; baseline in \textbf{bold}.}
\label{tab:ablation}
\footnotesize
\setlength{\tabcolsep}{5pt}
\renewcommand{\arraystretch}{1.1}
\resizebox{\textwidth}{!}{%
\begin{tabular}{l|c|cc|ccccc|cc}
\toprule
& \textbf{Default} & \multicolumn{2}{c|}{\textbf{Parent resampling}} & \multicolumn{5}{c|}{\textbf{Mutation kernel Choice}} & \multicolumn{2}{c}{\textbf{Breadth/depth}} \\
\cmidrule(lr){2-2} \cmidrule(lr){3-4} \cmidrule(lr){5-9} \cmidrule(lr){10-11}
& $N{=}8,K{=}2$ & uniform & greedy & diff & diff & rewrite & rewrite & uniform & $N{=}4$ & $N{=}16$ \\
& mixture & & (argmax) & no insp. & w/ insp. & no insp. & w/ insp. & 4-mix & $K{=}4$ & $K{=}1$ \\
\midrule
\textbf{Best} & \textbf{0.9993} & 0.9514 & 0.9760 & 0.9379 & 0.9868 & 0.9379 & 0.9514 & 0.9929 & 0.9379 & 0.9379 \\
\bottomrule
\end{tabular}%
}
\end{table}
\paragraph{Parent resampling.} {The adaptive resampling interpolates between uniform and greedy selection, transitioning from exploration to exploitation. To verify that this adaptivity is necessary, we replace it with two fixed alternatives at the extremes: \texttt{uniform} ignores the reward signal and \texttt{greedy} (argmax) always select the best. Both degrades, confirming that our adaptive choice is essential.}

\vspace{-.1in}
\paragraph{Mutation kernel choice.} {The mixture targets ergodicity, and Thompson sampling adaptively selects among kernels. We ablate each: replacing the mixture with single kernels, and Thompson with a uniform mixture. Every single kernel or uniform mixture degrades the baseline ($\leq 0.9929$), showing that both kernel diversity and adaptive selection are necessary.}

\vspace{-.1in}
\paragraph{Breadth vs.\ depth.} {Under a fixed budget of $N\times K=16$ LLM calls per iteration, the budget can favor more particles or longer MH chains. Both extremes collapse to $0.9379$: $N{=}4, K{=}4$ starves resampling of diversity; $N{=}16, K{=}1$ leaves chains too short to mix from the parent. The balanced $N{=}8, K{=}2$ split lets resampling explore broadly while each particle refines locally.}

%% file: sections/7_appendix.tex
\newpage
\appendix

\noindent
This appendix provides supplementary material for \textsc{SMCEvolve}. We first
position the method relative to prior work in Appendix~\ref{sec:related}, then
give the derivations and proofs supporting the SMC formulation
(Appendices~\ref{app:derivation}, \ref{app:MH}, \ref{app:weights},
\ref{app:kernel}, and~\ref{app:imported_complexity}), followed by
implementation details (Appendix~\ref{app:implementation}), benchmark/task
specifications (Appendix~\ref{app:task_details}), hyperparameters
(Appendix~\ref{app:hp}), prompts (Appendix~\ref{app:prompts}), and a
visualization of a representative run (Appendix~\ref{app:flow}).

\input{sections/6_related_work}

\section{Derivation of the Reward-Tilted Target Distribution}
\label{app:derivation}

We provide the full derivation of the reward-tilted target distribution in \Cref{def:ideal}.
Starting from the regularized variational objective in~\eqref{eq:variational}, the Lagrangian with the normalization constraint $\int p(x)\,dx = 1$ is
\begin{equation}
\mathcal{L}[p] = \int p(x)\, R(x)\, dx - \frac{1}{\beta}\int p(x) \log \frac{p(x)}{\pzero(x \mid q)}\, dx - \mu\!\left(\int p(x)\, dx - 1\right).
\end{equation}
Setting the functional derivative $\delta \mathcal{L} / \delta p(x) = 0$ gives
\begin{equation}
R(x) - \frac{1}{\beta}\left(\log \frac{p(x)}{\pzero(x \mid q)} + 1\right) - \mu = 0.
\end{equation}
Solving for $\log p(x)$ yields
\begin{equation}
\log p(x) = \beta R(x) + \log \pzero(x \mid q) - (1 + \mu\beta).
\end{equation}
Exponentiating and absorbing the $x$-independent term $e^{-(1 + \mu\beta)}$ into the partition function $Z(q)$ gives the target distribution $p^*(x \mid q)$ in~\eqref{eq:pstar}.

\section{MH construction yields a $p_t$-invariant kernel}
\label{app:MH}
Condition on the stage $t$ and context $\calC_t$, and write
$Q_t(\cdot\mid x,\calC_t)$ for the proposal kernel.
\begin{theorem}[MH construction yields a $p_t$-invariant kernel]
\label{thm:mh_invariance}
The MH kernel $M_t$ defined by~\eqref{eq:mh_ratio} is $p_t$-invariant, i.e., satisfies~\eqref{eq:invariance}.
\end{theorem}

\begin{proof}
We show that $M_t$ satisfies the stronger property of \emph{detailed balance},
\begin{equation}
\label{eq:detailed_balance}
p_t(x)\, M_t(x, y) \;=\; p_t(y)\, M_t(y, x), \qquad \forall\, x, y \in \mathcal{X},
\end{equation}
which immediately implies $p_t$-invariance: integrating both sides of~\eqref{eq:detailed_balance} over $x$ gives
\[
\int p_t(x)\, M_t(x, y)\, \mathrm{d}x
\;=\; p_t(y) \int M_t(y, x)\, \mathrm{d}x
\;=\; p_t(y),
\]
which is~\eqref{eq:invariance}. Intuitively, \eqref{eq:detailed_balance} says that at equilibrium $p_t$, the probability flux from $x$ to $y$ equals the flux from $y$ to $x$.

It remains to verify~\eqref{eq:detailed_balance} for the MH kernel. For $x = y$, both sides are equal, so the identity holds trivially. For $x \neq y$, we need to show
\[
p_t(x)\, Q_t(y \mid x,\calC_t)\, \alpha_t(x, y) \;=\; p_t(y)\, Q_t(x \mid y,\calC_t)\, \alpha_t(y, x).
\]
Without loss of generality, assume $p_t(x)\, Q_t(y \mid x,\calC_t) \leq p_t(y)\, Q_t(x \mid y,\calC_t)$ (if the opposite inequality holds, swap the roles of $x$ and $y$ in the argument below). Under this assumption:
\begin{itemize}[leftmargin=1.5em,topsep=0pt]
\item The ratio in~\eqref{eq:mh_ratio} satisfies $\dfrac{p_t(y)\, Q_t(x \mid y,\calC_t)}{p_t(x)\, Q_t(y \mid x,\calC_t)} \geq 1$, so $\alpha_t(x, y) = 1$.
\item The reverse ratio is $\dfrac{p_t(x)\, Q_t(y \mid x,\calC_t)}{p_t(y)\, Q_t(x \mid y,\calC_t)} \leq 1$, so $\alpha_t(y, x) = \dfrac{p_t(x)\, Q_t(y \mid x,\calC_t)}{p_t(y)\, Q_t(x \mid y,\calC_t)}$.
\end{itemize}
Now compute both sides:
\begin{align*}
\text{Left side:} \quad & p_t(x)\, Q_t(y \mid x,\calC_t)\, \alpha_t(x, y) \;=\; p_t(x)\, Q_t(y \mid x,\calC_t) \cdot 1 \;=\; p_t(x)\, Q_t(y \mid x,\calC_t). \\
\text{Right side:} \quad & p_t(y)\, Q_t(x \mid y,\calC_t)\, \alpha_t(y, x) \;=\; p_t(y)\, Q_t(x \mid y,\calC_t) \cdot \frac{p_t(x)\, Q_t(y \mid x,\calC_t)}{p_t(y)\, Q_t(x \mid y,\calC_t)} \;=\; p_t(x)\, Q_t(y \mid x,\calC_t).
\end{align*}
Both sides agree, so detailed balance holds, and hence $p_t$-invariance. \qedhere
\end{proof}

\section{Weight Simplification under Time-Reversal Backward Kernel}
\label{app:weights}

Starting from the general SMC weight:
\begin{equation}
w_t(x_{t-1}, x_t) = \frac{\gamma_t(x_t)\, L_{t-1}(x_t, x_{t-1})}{\gamma_{t-1}(x_{t-1})\, M_t^K(x_{t-1}, x_t)}.
\end{equation}
Substituting the time-reversal backward kernel $L_{t-1}(x_t, x_{t-1}) = p_t(x_{t-1})\, M_t^K(x_{t-1}, x_t) / p_t(x_t)$:
\begin{align}
w_t(x_{t-1}, x_t) &= \frac{\gamma_t(x_t)}{\gamma_{t-1}(x_{t-1})} \cdot \frac{p_t(x_{t-1})\, M_t^K(x_{t-1}, x_t)}{p_t(x_t)\, M_t^K(x_{t-1}, x_t)} \\
&= \frac{\gamma_t(x_t)}{\gamma_{t-1}(x_{t-1})} \cdot \frac{\gamma_t(x_{t-1}) / Z_t}{\gamma_t(x_t) / Z_t} \\
&= \frac{\gamma_t(x_{t-1})}{\gamma_{t-1}(x_{t-1})} = e^{\Delta\beta_t \cdot R(x_{t-1})}.
\end{align}
Both $M_t^K$ and $L_{t-1}$ cancel completely, confirming~\eqref{eq:weight}.

\section{Realizing the Forward Kernel with LLM Proposals}
\label{app:kernel}

Algorithm~\ref{alg:smc} requires a $p_t$-invariant MCMC kernel $M_t$ at each step. {Here $Q_t$ is the LLM proposal distribution, while one application of $M_t$ is the proposal plus MH accept/reject transition; Algorithm~\ref{alg:smc} applies this transition $K$ times per particle.} We now detail how this kernel is realized when the only available sampler is a black-box LLM.

\paragraph{LLM as proposal distribution.}
At step $t$, given a parent program $x$ selected by resampling, the LLM proposal kernel generates a candidate program $x'$ by sampling from
\begin{equation}\label{eq:proposal}
x' \sim Q_t(\cdot \mid x,\calC_t),
\end{equation}
where $\calC_t$ denotes the \emph{context} at step $t$: a structured collection of information provided to the LLM alongside the parent program $x$ and task $q$. This context may include the history of previously evaluated programs with their rewards, high-performing programs from other particles (inspirations), and instructions specifying the edit granularity. The notation $Q_t(\cdot \mid x,\calC_t)$ emphasizes that the proposal is implemented by the \emph{same LLM} that defines the prior $p_0(\cdot\mid q)$, but with additional conditioning on the parent and context. When $\calC_t = \varnothing$ and no parent is provided, the proposal reduces to the prior $p_0(\cdot \mid q)$.

\paragraph{The invariance requirement.}
To ensure that $M_t$ leaves $p_t$ invariant, the standard approach is Metropolis--Hastings (MH): accept $x'$ with probability
\begin{equation}\label{eq:mh_ratio}
\alpha_t(x, x') = \min\!\bigg\{1,\; \frac{p_t(x')\, Q_t(x \mid x',\calC_t)}{p_t(x)\, Q_t(x' \mid x,\calC_t)}\bigg\}.
\end{equation}
Expanding $p_t(x) \propto p_0(x \mid q)\, e^{\beta_t R(x)}$, the acceptance ratio becomes
\begin{equation}\label{eq:mh_expanded}
\alpha_t(x, x') = \min\!\bigg\{1,\; e^{\beta_t(R(x') - R(x))} \cdot \underbrace{\frac{p_0(x' \mid q)}{p_0(x \mid q)} \cdot \frac{Q_t(x \mid x',\calC_t)}{Q_t(x' \mid x,\calC_t)}}_{\displaystyle \sigma_t(x, x')}\bigg\}.
\end{equation}
The composite ratio $\sigma_t(x, x')$ captures the total structural asymmetry between forward and reverse transitions. Computing it requires evaluating both the prior density ratio and the proposal density ratio, both of which are intractable for black-box LLMs.

\paragraph{From the full MH ratio to a tractable approximation.}
Computing $\sigma_t(x,x')$ requires evaluating the LLM prior ratio and the
proposal density ratio, both of which are unavailable through black-box
LLM APIs. To obtain a tractable acceptance rule, we adopt the reward-only
criterion in~\Cref{eq:mh_simple} of the main text, which corresponds to
substituting $\sigma_t(x,x') = 1$ in~\eqref{eq:mh_expanded}. While this approximation may cause the mutation kernel to deviate from exact $p_t$-invariance, it is motivated by the heuristic that the LLM tends
to perform local modifications---changing a few lines, tuning a
hyperparameter, refactoring a subroutine. For such edits, both the prior ratio
$p_0(x' \mid q)/p_0(x \mid q)$ and the proposal ratio
$Q_t(x \mid x',\calC_t)/Q_t(x' \mid x,\calC_t)$ are close to one,
and even when each factor deviates from unity their product can stay
close to one if the deviations partially cancel. The high MH acceptance rates observed in
\Cref{fig:kernel} provide indirect empirical evidence that the
substitution is a reasonable approximation in practice.

\section{Proof of \Cref{thm:complexity}}
\label{app:imported_complexity}

The proof follows the proof sketch in \Cref{sec:convergence}: we state the
imported finite-sample SMC bound of \citep[Theorem~1 \&~Sec.~5.1]{marion2023finite}
in its mixing-time form (which does not require kernel reversibility) and
then derive the three substitutions (S1), (S2), (S3) that map its abstract
quantities to the \textsc{SMCEvolve} variables.

\paragraph{Warm-start mixing time.}
For a Markov kernel $M$ with invariant distribution $\mu$, an
$\omega$-warm starting set
$P_\omega(\mu) \triangleq \{\nu : \sup_{B}\, \nu(B)/\mu(B) \le \omega\}$,
and an accuracy $\eta \in (0,1)$, the warm-start mixing time of $M$ is
\begin{equation}
\label{eq:mixing_time_def}
\tau_M(\eta, \omega) \;\triangleq\; \min\!\left\{
K \ge 1 :\; \sup_{\nu \in P_\omega(\mu)} \|\nu M^K - \mu\|_{\mathrm{TV}}
\;\le\; \eta
\right\}.
\end{equation}
This is Definition (Sec.~2.1) of \citep{marion2023finite}.

\paragraph{Imported bound \citep[Theorem~1 \&~Sec.~5.1]{marion2023finite}.}
Consider an SMC sampler over a geometric-mixture path
\(\mu_0,\ldots,\mu_T=\pi\) with $T = \lceil \log\Gamma \rceil$ and
$\beta_t = t/T$, initialized with iid samples from \(\mu_0\) and applying
\(K\) transitions of a \(\mu_t\)-invariant Markov kernel \(M_t\) at each
stage \(t\). Suppose
\begin{enumerate}[leftmargin=1.5em,topsep=0pt]
\item the bridge weights satisfy $\sup_x w_t(x) \le W$ and
$z_{t-1}/z_t \le Z$ for every $t$, with $\gamma \triangleq W \cdot Z$
(this is AS1 of \citealp{marion2023finite});
\item there exists $\tau < \infty$ such that
$\tau_{M_t}(\eta_\star, \omega_\star) \le \tau$ for every $t$, at the
internal precision/warmness level chosen by the analysis (Theorem~1 of
\citealp{marion2023finite} uses $\omega_\star = 2$ and
$\eta_\star = 1/(8NT)$);
\item $\Gamma \triangleq \sup_x d\pi/d\mu_0(x) < \infty$.
\end{enumerate}
Then for every \(\|f\|_\infty\le 1\) and \(\epsilon>0\), there exists a
choice of \((N,T,K)\) such that the SMC empirical measure $\eta_T^N$
satisfies \(\mathbb P(|\eta_T^N(f)-\pi(f)|\le\epsilon)\ge 3/4\), and the
total kernel-evaluation budget is bounded by
\begin{equation}
\label{eq:imported}
\calB \;=\; NTK \;=\; \mathcal{O}^\ast\!\left(
(\epsilon^{-2}\vee \gamma^2) \cdot \tau \cdot
\log\Gamma\,\log^2\!\log\Gamma
\right).
\end{equation}
The reversible-kernel form \citep[Cor.~6.1]{marion2023finite} is
recovered by the spectral-gap-to-mixing-time bound
$\tau_{M_t}(\eta, \omega) \le (1/\rho^{\mathrm{gap}})\big[\log(2/\eta) + \log(\omega - 1)\big]$,
which gives $\tau = \mathcal{O}(1/\rho^{\mathrm{gap}})$ up to log factors
absorbed into $\mathcal{O}^\ast$.

\paragraph{Setup for \textsc{SMCEvolve}.}
The bridges $p_t(x)\propto p_0(x\mid q) e^{\beta_t R(x)}$ with
$0=\beta_0<\cdots<\beta_T=\beta$ form a geometric-mixture path from
$p_0(\cdot\mid q)$ to $p^*(\cdot\mid q)$. The initialization condition holds
because particles are iid from $p_0(\cdot\mid q)$
(\Cref{alg:smc}, line 2), and the kernel correctness/mixing condition holds
under the kernel hypothesis of \Cref{thm:complexity}. It remains to identify $(\gamma^2, \tau, \log\Gamma)$ in
\eqref{eq:imported} with \textsc{SMCEvolve} quantities. We do so in
three steps.

\subsection*{(S1) Bridge regularity $\gamma^2 \le \kappa^{-1}$}

The imported bound~\eqref{eq:imported} requires a deterministic upper bound
on the adjacent $L^2$-bridge ratio, which we adopt as part of the
hypothesis of \Cref{thm:complexity}:
\begin{equation}
\left\|\frac{p_t}{p_{t-1}}\right\|^2_{L^2(p_{t-1})} \;\le\; \kappa^{-1},
\qquad t = 1,\dots,T.
\label{eq:S1-app}
\end{equation}
The parameterization in terms of $\kappa$ is justified by two
considerations.

\paragraph{Worst-case schedule-dependent bound.}
The geometric reward-tilted bridge~\eqref{eq:bridging} satisfies
\begin{equation}
\left\|\frac{p_t}{p_{t-1}}\right\|_\infty
\le \frac{e^{\Delta\beta_t R_+}}{p_{t-1}(e^{\Delta\beta_t R})}
\;\le\; e^{\Delta\beta_t(R_+ - R_-)}
\;=\; e^{\Delta\beta_t \Delta_R},
\label{eq:bridge-linfty}
\end{equation}
where the second inequality uses
$p_{t-1}(e^{\Delta\beta_t R}) \ge e^{\Delta\beta_t R_-}$. Since the $L^2$
norm is dominated by the $L^\infty$ norm, this gives the deterministic
certificate
$\kappa^{-1} \le e^{2\Delta\beta_{\max} \Delta_R}$, with
$\Delta\beta_{\max} \triangleq \max_t \Delta\beta_t$. Short bridges
(small $\Delta\beta_t$) imply small bridge ratios.

\paragraph{ESS bisection as a finite-sample mechanism that targets~\eqref{eq:S1-app}.}
The empirical ESS
\begin{equation}
\widehat{\mathrm{ESS}}_t = \frac{\big(\sum_{n=1}^N w_t^{(n)}\big)^2}{\sum_{n=1}^N (w_t^{(n)})^2},
\qquad
w_t^{(n)} = \frac{p_t(x_{t-1}^{(n)})}{p_{t-1}(x_{t-1}^{(n)})}
\label{eq:ess-def-app}
\end{equation}
is the finite-sample analogue of $\|p_t/p_{t-1}\|^{-2}_{L^2(p_{t-1})}$:
when $\{x_{t-1}^{(n)}\}_{n=1}^N$ is an iid sample from $p_{t-1}$, a law of
large numbers gives
$\widehat{\mathrm{ESS}}_t / N \to \|p_t/p_{t-1}\|^{-2}_{L^2(p_{t-1})}$ as
$N\to\infty$. The ESS bisection rule~\eqref{eq:ess} is the algorithmic
mechanism that targets~\eqref{eq:S1-app} by enforcing its empirical
analogue $\widehat{\mathrm{ESS}}_t / N \ge \kappa$. We do not turn this
targeting into a formal derivation, since the SMC particles at stage
$t-1$ are not iid from $p_{t-1}$ and the schedule
$\{\lambda_t\}_{t=0}^T$ is data-dependent;~\eqref{eq:S1-app} is therefore
treated as an explicit assumption (with worst-case certificate
$e^{2\Delta\beta_{\max} \Delta_R}$) rather than a consequence of the ESS
rule.

\subsection*{(S2) Mixing-time bound $\tau \le 1/(1-\rho)$ via uniform ergodicity}

The kernel hypothesis of \Cref{thm:complexity} provides uniform ergodicity
$\sup_x \|M_t^K(x,\cdot) - p_t\|_{\mathrm{TV}} \le C \rho^K$ with rate
$\rho \in (0,1)$. We show that this directly yields a warm-start
mixing-time bound that feeds into~\eqref{eq:imported} without any
reversibility hypothesis. The argument is in four steps.

\paragraph{(i) TV convexity in the first argument.}
For any starting distribution $\nu$ and any measurable $A$, since $\nu$
is a probability measure,
\[
\nu M_t^K(A) - p_t(A)
\;=\; \int [M_t^K(x, A) - p_t(A)]\, \nu(dx).
\]
Taking absolute values, then $\sup_A$, and exchanging
$\sup_A \int \le \int \sup_A$ gives
\[
\|\nu M_t^K - p_t\|_{\mathrm{TV}}
\;\le\; \int \|\delta_x M_t^K - p_t\|_{\mathrm{TV}}\, \nu(dx)
\;\le\; \sup_x \|\delta_x M_t^K - p_t\|_{\mathrm{TV}}
\;\le\; C \rho^K.
\]
In particular this holds for every $\omega$-warm
$\nu \in P_\omega(p_t)$.

\paragraph{(ii) Inverting.}
Setting $C \rho^K \le \eta$ yields $K \ge \log(C/\eta) / \log(1/\rho)$,
so the warm-start mixing time of $M_t$ is bounded by
\begin{equation}
\tau_{M_t}(\eta, \omega)
\;\le\; \frac{\log(C/\eta)}{\log(1/\rho)},
\qquad \forall\, \omega \ge 1, \; \eta \in (0,1).
\label{eq:tv-to-tau}
\end{equation}

\paragraph{(iii) Elementary inequality.}
For $\rho \in (0,1)$, $\log(1/\rho) \ge 1 - \rho$: the function
$f(\rho) = -\log\rho - (1 - \rho)$ has $f(1) = 0$ and
$f'(\rho) = -1/\rho + 1 < 0$ on $(0,1)$, so $f \ge 0$ on $(0,1)$.
Combining with~\eqref{eq:tv-to-tau} gives
\begin{equation}
\tau_{M_t}(\eta, \omega) \;\le\; \frac{\log(C/\eta)}{1 - \rho}.
\label{eq:S2-app}
\end{equation}

\paragraph{(iv) Substitution into~\eqref{eq:imported}.}
Theorem~1 of \citep{marion2023finite} sets
$(\eta_\star, \omega_\star) = (1/(8NT), 2)$, so the mixing-time factor
in~\eqref{eq:imported} satisfies
\[
\tau \;=\; \max_t \tau_{M_t}(1/(8NT), 2)
\;\le\; \frac{\log(8NTC)}{1 - \rho}
\;=\; \frac{\mathcal{O}(\log(NT))}{1 - \rho}.
\]
The $\mathcal{O}(\log(NT))$ factor is polylogarithmic and is absorbed
into $\widetilde{\mathcal{O}}$ in the SMCEvolve rate~\eqref{eq:main-bound}.

\subsection*{(S3) $\log\Gamma \le \beta\Delta_R$ via reward oscillation}

\begin{lemma}[Path length for reward tilting]
\label{lem:path_length_reward}
Under the kernel hypothesis of \Cref{thm:complexity}, suppose $R$ is bounded
on the support of $p_0(\cdot\mid q)$ with oscillation $\Delta_R$. Then
\(\log\Gamma\le \beta\Delta_R\).
\end{lemma}

\begin{proof}
Since $p^*(x\mid q) = Z(q)^{-1} p_0(x\mid q) e^{\beta R(x)}$,
\[
\Gamma
= \sup_{x:p_0(x\mid q)>0} \frac{p^*(x\mid q)}{p_0(x\mid q)}
= \sup_{x:p_0(x\mid q)>0} \frac{e^{\beta R(x)}}{Z(q)}
\le \frac{e^{\beta R_+}}{Z(q)}.
\]
Moreover,
\[
Z(q) = \E_{p_0}[e^{\beta R}] \;\ge\; e^{\beta R_-}
\]
because $R(x) \ge R_-$ on the support of $p_0$. Hence
$\Gamma \le e^{\beta(R_+ - R_-)} = e^{\beta\Delta_R}$, which gives
$\log\Gamma \le \beta\Delta_R$.
\end{proof}

\paragraph{Combining (S1)--(S3).}
Substituting \eqref{eq:S1-app}, \eqref{eq:S2-app}, and
\Cref{lem:path_length_reward} into the imported bound \eqref{eq:imported}:
\[
\calB \;=\; \mathcal{O}^\ast\!\left(
\frac{\epsilon^{-2}\vee \kappa^{-1}}{1-\rho}\,
\log\Gamma\,\log^2\!\log\Gamma
\right)
\;\subseteq\;
\widetilde{\mathcal{O}}\!\left(
\frac{\epsilon^{-2}\vee \kappa^{-1}}{1-\rho}\,
\beta \Delta_R
\right),
\]
where the iterated-logarithm factors are polylogarithmic in
$\beta\Delta_R$ and absorbed into $\widetilde{\mathcal{O}}$. The
concentration statement
$\mathbb P(|\eta_T^N(f)-p^*(f)|\le \epsilon) \ge 3/4$ is inherited
directly from the imported bound. This is exactly
\eqref{eq:concentration}--\eqref{eq:main-bound} in \Cref{thm:complexity}.
\qed

\subsection{Per-factor requirements and degenerate regimes}
\label{app:regime}

\Cref{thm:complexity} bounds the joint LLM-call budget $\calB = NTK$,
but each factor must individually be large enough for the SMC sampler
to function. Concretely, the theorem requires the particle count to
satisfy $N \gtrsim \kappa^{-1} \vee \epsilon^{-2}$, so the ESS
constraint controls the bridge-ratio variance and supports the Monte
Carlo error $\epsilon$, and the mutation count to satisfy
$K \gtrsim \log(1/\epsilon)/(1-\rho)$, so the uniform-ergodicity TV
error $C\rho^K$ falls below $\epsilon$ at every bridge; the annealing
length $T = \widetilde{\mathcal{O}}(\beta\Delta_R)$ is set adaptively
by the ESS bisection rule and is not a free hyperparameter. The
extreme cases $N = 1$ (no resampling diversity) and $K = 1$ (a single
MH step per particle, insufficient mixing whenever $\rho$ is close
to~$1$) sit outside this regime, and the formal guarantee no longer
applies.

\section{Implementation Details}
\label{app:implementation}

\subsection{Island-Based Parallelism}
\label{sec:islands}

\textsc{SMCEvolve} runs $I$ independent SMC chains (\emph{islands}), each with $N$ particles and its own temperature schedule. Every $\tau$ epochs, islands exchange top-performing particles via a migration step: each island sends its best $m$ particles to a randomly assigned neighbor and retains the top $N$ after merging. This maintains diversity across subpopulations while enabling parallel execution.

\subsection{Resampling and Numerical Stability}
\label{app:resampling}

We use \emph{systematic resampling}~\citep{carpenter1999improved}, which
requires only a single uniform random number and is computationally convenient
in practice; see \citep{douc2005comparison} for a comparison with the
stratified variant of \citep{kitagawa1996monte} and a variance analysis. For numerical stability, the log-weights
$\log w_t^{(n)} = \Delta\beta_t \cdot R(x_{t-1}^{(n)})$ are shifted by their
maximum before exponentiation (the log-sum-exp trick), preventing overflow
without affecting the normalized weights~\eqref{eq:softmax}.

\subsection{Summary of Theoretical Guarantees}
\label{app:guarantees}

\Cref{tab:guarantees} summarizes how each component of \textsc{SMCEvolve} maps to the SMC framework and which guarantees are preserved.

\begin{table}[h]
\caption{Mapping between the SMC framework and the \textsc{SMCEvolve} implementation. \textbf{Exact}: the theoretical guarantee is preserved. \textbf{Approx.}: a practical relaxation is introduced.}
\label{tab:guarantees}
\centering
\small
\begin{tabular}{llc}
\toprule
\textbf{Component} & \textbf{Implementation} & \textbf{Status} \\
\midrule
Target distribution $p^* \propto p_0\, e^{\beta R}$ & Same & Exact \\
Annealing path $p_t \propto p_0\, e^{\beta_t R}$ & Same & Exact \\
Importance weights $w_t = e^{\Delta\beta_t R(x_{t-1})}$ & Log-sum-exp stabilized & Exact \\
Temperature schedule & ESS bisection & Adaptive extension \\
Resampling & Systematic & Exact \\
Forward kernel $M_t$ & {MH kernel with LLM proposal mixture} & Approx. \\
Kernel selection & Thompson sampling mixture & Approx. (via $\widehat M_t^{H_t}$) \\
Parallelism & Islands + migration & Approx. \\
\bottomrule
\end{tabular}
\end{table}

The target distribution, annealing path, and importance weights remain exact.
Mutation heuristics, adaptive kernel selection, adaptive temperature schedules,
and island-based parallel extensions enter the theory only through additional
approximation terms or extra assumptions beyond the main theorem section.

\section{Task Details}
\label{app:task_details}

This appendix collects natural-language descriptions of every benchmark
task evaluated in Tables~\ref{tab:math}--\ref{tab:algotune} and
Figure~\ref{fig:tinystories-curve}. Formal interfaces (input/output
signatures, evaluators) are documented in the per-task \texttt{task.md}
files released with our code.

\subsection{Math domain (Table~\ref{tab:math})}
\label{app:tasks_math}

Ten open problems from analytic combinatorics and discrete geometry,
ported from the public release accompanying DeepMind's AlphaEvolve. Each
task asks the program to \emph{construct} a concrete mathematical object
(a packing, a point set, or a step function) that pushes a known numerical
constant in the direction predicted by theory; a score of $1.0$ ties the
2024 AlphaEvolve benchmark, and scores above $1.0$ strictly improve on it.

\begin{table}[h]
\centering
\small
\caption{Math domain --- AlphaEvolve combinatorial geometry tasks.}
\label{tab:tasks_math}
\setlength{\extrarowheight}{6pt}
\begin{tabularx}{\linewidth}{@{}>{\raggedright\arraybackslash}p{0.27\linewidth}X@{}}
\toprule
\textbf{Task} & \textbf{Description} \\
\midrule
Circle Packing in Rect.\ ($N{=}21$)
  & Place 21 disjoint circles in an axis-aligned rectangle of perimeter
    $\le 4$ (width $+$ height $\le 2$); maximize the sum of radii.
    AlphaEvolve App.\ B.13. \\
Hexagon Packing ($N{=}11$)
  & Pack 11 disjoint unit regular hexagons inside a larger regular hexagon;
    minimize the outer side length (equivalently, maximize its reciprocal).
    AlphaEvolve App.\ B.7. \\
Heilbronn Convex ($N{=}13$)
  & Inside a convex region of unit area, place 13 points so that the
    smallest triangle spanned by any three is as large as possible.
    AlphaEvolve App.\ B.10. \\
Heilbronn Triangle ($N{=}11$)
  & Same min-triangle-area objective restricted to the unit equilateral
    triangle with vertices $(0,0),(1,0),(\tfrac12,\tfrac{\sqrt3}{2})$.
    AlphaEvolve App.\ B.9. \\
Min--Max--Min Dist.\ ($n{=}16,d{=}2$)
  & Place 16 points in the plane; minimize the ratio of the largest to
    smallest pairwise Euclidean distance (a ``close-to-uniform''
    configuration). AlphaEvolve App.\ B.8. \\
Kissing Number ($d{=}11$)
  & Construct an integer lattice cloud in $\mathbb{R}^{11}$ whose minimum
    pairwise squared distance is at least the maximum squared norm of any
    point; maximize cardinality (a lower bound on the 11D kissing number).
    AlphaEvolve App.\ B.11. \\
1st Autocorrelation Inequality
  & Construct a non-negative step function $f:\mathbb{R}\to\mathbb{R}_{\ge0}$
    whose self-convolution improves the upper bound on the constant $C_1$.
    AlphaEvolve App.\ B.1. \\
2nd Autocorrelation Inequality
  & Construct a non-negative step function $f$ that pushes the lower bound
    on the related constant $C_2$ (defined via the autoconvolution norm)
    upward. AlphaEvolve App.\ B.2. \\
3rd Autocorrelation Inequality
  & Construct a (possibly signed) step function $f$ that improves the
    upper bound on $C_3$, governing $|f \star f|$.
    AlphaEvolve App.\ B.3. \\
Erd\H{o}s Minimum Overlap
  & Construct a non-negative step function $h$ satisfying the Minimum
    Overlap regularity constraints, lowering the best known upper bound on
    its asymptotic constant. AlphaEvolve App.\ B.5. \\
\bottomrule
\end{tabularx}
\end{table}

\subsection{Symbolic Regression domain (Table~\ref{tab:sr})}
\label{app:tasks_sr}

Sixteen scientific-equation discovery tasks (4 splits $\times$ 4 tasks)
drawn from LLM-SRBench~\citep{shojaee2025llm}. Each task hides a
ground-truth analytical expression behind a small numeric dataset; the
program must propose a closed-form parametric expression (with up to ten
free coefficients) whose mean-squared error against the held-out training
data is as small as possible. The reward $-\log_{10}(\mathrm{MSE})$
rewards exponentially-better fits, so a one-decade improvement in MSE
counts as $+1.0$ in score.

\begin{table}[h]
\centering
\small
\caption{Symbolic Regression domain --- four LLM-SRBench splits, four
tasks per split.}
\label{tab:tasks_sr}
\setlength{\extrarowheight}{6pt}
\begin{tabularx}{\linewidth}{@{}>{\raggedright\arraybackslash}p{0.22\linewidth}X@{}}
\toprule
\textbf{Split} & \textbf{Description} \\
\midrule
\texttt{bio\_pop\_growth}
  & Population-growth ODEs for $P(t)$ (logistic, Gompertz, Allee-effect,
    and similar). Given samples $(t,P)$, the program must express the
    instantaneous growth rate $\mathrm{d}P/\mathrm{d}t$ in closed form.
    Tasks BPG0--BPG3 cover four ODEs of varying complexity. \\
\texttt{chem\_react}
  & Chemical reaction-rate ODEs for the concentration $A(t)$ of a species
    (first-order decay, autocatalytic, Michaelis--Menten-like, etc.). Given
    $(t,A)$ samples, the program must recover $\mathrm{d}A/\mathrm{d}t$.
    Tasks CRK0--CRK3 cover four mechanisms. \\
\texttt{matsci}
  & Materials-science constitutive laws $\sigma(\varepsilon,T)$ from
    continuum mechanics (linear elasticity with thermal expansion,
    power-law hardening, Arrhenius softening, etc.). Tasks MatSci0--MatSci3
    cover four laws. \\
\texttt{phys\_osc}
  & Second-order oscillator ODEs for the velocity $v(t)$ of a particle
    (harmonic, damped, driven, anharmonic). The program expresses the
    acceleration $\mathrm{d}v/\mathrm{d}t$ from a task-dependent subset of
    $(x,t,v)$: PO0/PO2 use all three; PO1 uses $(x,t)$; PO3 uses $(t,v)$. \\
\bottomrule
\end{tabularx}
\end{table}

\subsection{AlgoTune domain (Table~\ref{tab:algotune})}
\label{app:tasks_algotune}

Eight numerical, signal-processing, and linear-algebra tasks from the
public AlgoTune benchmark~\citep{press2025}. Each task ships with a known
\emph{correct} reference implementation (typically a thin wrapper around
SciPy or LAPACK); the program's job is to produce a drop-in replacement
that is measurably faster on the benchmark's input distribution while
still passing AlgoTune's correctness verifier. The reward is the
wall-clock speedup over the reference (so $1.0$ $=$ no speedup, $10.0$
$=$ ten-times faster); any incorrect output scores $0$.

\begin{table}[h]
\centering
\small
\caption{AlgoTune domain --- runtime-optimization tasks.}
\label{tab:tasks_algotune}
\setlength{\extrarowheight}{6pt}
\begin{tabularx}{\linewidth}{@{}>{\raggedright\arraybackslash}p{0.30\linewidth}X@{}}
\toprule
\textbf{Task} & \textbf{Description} \\
\midrule
\texttt{affine\_transform\_2d}
  & Apply a 2D affine warp (encoded as a $2{\times}3$ matrix) to an
    $n{\times}n$ image with cubic-spline interpolation and zero padding.
    Reference: \texttt{scipy.ndimage.affine\_transform}. \\
\texttt{convolve2d\_full\_fill}
  & Compute the full 2D convolution of two real matrices of size
    $(30n,30n)$ and $(8n,8n)$ with zero-padded boundaries --- the dense
    direct case where blocking, FFT, or specialized kernels yield large
    speedups. \\
\texttt{eigenvectors\_complex}
  & Full eigendecomposition of a real square matrix that may have
    complex-conjugate pairs; return eigenvalues sorted by descending real
    then imaginary part with unit-norm eigenvectors. Reference:
    \texttt{numpy.linalg.eig}. \\
\texttt{fft\_cmplx\_scipy\_fftpack}
  & N-dimensional FFT of a complex $n{\times}n$ matrix matching the
    \texttt{scipy.fftpack} reference output. \\
\texttt{fft\_convolution}
  & Discrete convolution $z[n]=\sum_k x[k]\,y[n{-}k]$ of two 1D signals
    via FFT, supporting the standard \texttt{full}/\texttt{same}/%
    \texttt{valid} output modes. \\
\texttt{lu\_factorization}
  & LU factorization $A=PLU$ of a square matrix; return the permutation
    $P$, unit-lower-triangular $L$, and upper-triangular $U$. Reference:
    \texttt{scipy.linalg.lu}. \\
\texttt{polynomial\_real}
  & Given coefficients of a degree-$n$ polynomial with all real roots,
    return the roots in descending order, accurate to at least three
    decimal places. \\
\texttt{psd\_cone\_projection}
  & Project a symmetric matrix $A$ onto the cone of symmetric PSD matrices
    (zero out negative eigenvalues). The closed form is short, but a naive
    implementation leaves several-times speedup on the table. \\
\bottomrule
\end{tabularx}
\end{table}

\subsection{AutoResearch domain (Figure~\ref{fig:tinystories-curve})}
\label{app:tasks_autoresearch}

A single high-stakes systems task ported from
\texttt{karpathy/autoresearch}~\citep{karpathy_autoresearch_2026}. The
program under evolution is a \emph{complete} single-file GPT pretraining
script: model architecture, tokenizer setup, optimizer, learning-rate
schedule, batching, training loop, and final evaluation. Each candidate
trains for 60\,s of wall-clock time on a single RTX A5000 (24\,GB) on the
climbmix corpus, after which the script must report \texttt{val\_bpb}
(bits-per-byte on a held-out shard). Lower \texttt{val\_bpb} means a
better language model; the reward
$\max(0,\,2.0-\texttt{val\_bpb})$ is anchored so that a random baseline
($\texttt{val\_bpb}\approx 2.6$) scores $0$ and a well-trained small GPT
($\texttt{val\_bpb}\approx 1.0$) scores about $1.0$.

\begin{table}[h]
\centering
\small
\caption{AutoResearch domain --- single-GPU GPT pretraining task.}
\label{tab:tasks_autoresearch}
\setlength{\extrarowheight}{6pt}
\begin{tabularx}{\linewidth}{@{}>{\raggedright\arraybackslash}p{0.25\linewidth}X@{}}
\toprule
\textbf{Field} & \textbf{Specification} \\
\midrule
Goal
  & Evolve a single-file GPT pretraining script that minimizes
    \texttt{val\_bpb} on a held-out shard within a fixed compute budget. \\
Reward
  & $\max(0,\,2.0-\texttt{val\_bpb})$. Random baseline $\approx 0$;
    well-trained small GPT $\approx 1.0$. \\
Hardware / time
  & Single RTX A5000 (24\,GB), 60\,s wall-clock training per candidate
    ($\approx$90\,s end-to-end including evaluation). \\
Hard constraints
  & VRAM $\le$ 24\,GB (OOM scores $0$); single GPU (no distributed
    training); no extra Python packages may be installed. \\
Seed program
  & 6-layer GPT, embedding dim $384$, $3$ attention heads, full causal
    attention via PyTorch \texttt{scaled\_dot\_product\_attention};
    MuonAdamW optimizer (Muon for 2D weight matrices, AdamW elsewhere). \\
Search space
  & Free to modify model depth/width, optimizer, attention backend, batch
    size, learning-rate schedule, initialization, etc. The evaluation
    harness (tokenizer, dataloader, \texttt{evaluate\_bpb} from the fixed
    \texttt{prepare.py}) must be left untouched, and the last stdout line
    must be the literal \texttt{val\_bpb: <float>} produced by it. \\
\bottomrule
\end{tabularx}
\end{table}

\section{Existing Assets and Licenses}
\label{app:licenses}

We use existing assets only as benchmark and evaluation resources. The
math tasks are ported from the public AlphaEvolve benchmark release; the
symbolic-regression tasks are drawn from LLM-SRBench; the numerical
optimization tasks are drawn from AlgoTune; and the AutoResearch task uses
the TinyStories/climbmix pretraining setup from \texttt{karpathy/autoresearch}.
We credit the original sources in the main text and task descriptions. We
do not redistribute third-party datasets, pretrained models, or benchmark
code as new assets; our code release will include attribution, links, and
the applicable license or terms-of-use information for each external
resource. Commercial LLM APIs are accessed under the providers' terms, and
users reproducing the experiments must supply their own API credentials.

\section{Hyperparameters}
\label{app:hp}

\Cref{tab:hp} lists the hyperparameters that materially affect
\textsc{SMCEvolve}'s behaviour, together with the values used in our
experiments. The compute-budget group sets the LLM-call ceiling
$\le n_{\text{islands}}\!\times\!N\!\times\!K\!\times\!T_{\max}$; the
SMC-annealing group controls the adaptive temperature schedule; the
remaining groups govern mixing, mutation, and LLM behaviour.

\begin{table}[h]
\centering
\small
\caption{\textsc{SMCEvolve} hyperparameters.}
\label{tab:hp}
\renewcommand{\arraystretch}{1.15}
\begin{tabularx}{\linewidth}{@{}>{\ttfamily\raggedright\arraybackslash}l c X@{}}
\toprule
\normalfont\textbf{Parameter} & \normalfont\textbf{Value} & \normalfont\textbf{Meaning} \\
\midrule
\multicolumn{3}{@{}l}{\normalfont\textit{Compute budget}} \\
n\_islands               & $2$    & \normalfont Number of parallel islands. \\
particles\_per\_island   & $8$    & \normalfont Particles $N$ per island. \\
n\_proposals             & $2$    & \normalfont $K$ sequential LLM proposals per particle per step, used as a $K$-step MH chain. \\
min\_iterations          & $3$    & \normalfont Lower bound on SMC steps; also caps $\Delta\lambda\!\le\!1/\text{min\_iterations}$. \\
max\_iterations          & $15$   & \normalfont Hard upper bound $T_{\max}$ on SMC steps. \\
\addlinespace[2pt]\midrule
\multicolumn{3}{@{}l}{\normalfont\textit{SMC annealing}} \\
beta                     & $20$   & \normalfont Target inverse temperature $\beta$ (annealing endpoint). \\
kappa                    & $0.9$  & \normalfont ESS threshold used by bisection to pick $\lambda_t$. \\
\addlinespace[2pt]\midrule
\multicolumn{3}{@{}l}{\normalfont\textit{Island migration}} \\
migration\_interval      & $3$    & \normalfont SMC steps between successive migration events. \\
migration\_size          & $1$    & \normalfont Particles exchanged per migration event. \\
\addlinespace[2pt]\midrule
\multicolumn{3}{@{}l}{\normalfont\textit{Mutation kernel}} \\
kernel\_selection        & adaptive & \normalfont Thompson-sampling mixture over the $2{\times}2$ kernel grid (\texttt{diff}/\texttt{rewrite}~$\times$~with/without inspirations). \\
top\_k\_inspiration      & $2$    & \normalfont Highest-reward exemplars from the same island. \\
diverse\_inspirations    & $2$    & \normalfont MAP-Elites-style diversity exemplars (embedding-distance). \\
\addlinespace[2pt]\midrule
\multicolumn{3}{@{}l}{\normalfont\textit{LLM backend}} \\
temperature              & $1.0$  & \normalfont LLM sampling temperature. \\
max\_tokens              & $4096$ & \normalfont Maximum tokens per LLM response. \\
models                   & 50/50  & \normalfont Ensemble of \texttt{gpt-5-mini} and \texttt{gemini-3-flash} (\texttt{gpt-5.4}~+~\texttt{gemini-3-pro} for AutoResearch). \\
\bottomrule
\end{tabularx}
\end{table}

\section{Broader Impact}
\label{app:broader_impact}

\textsc{SMCEvolve} is intended as a research framework for LLM-driven
program search. Its main potential benefit is to make automated scientific
discovery and program optimization more sample-efficient and easier to
monitor through explicit resampling and stopping criteria. Our experiments
are conducted on benchmark tasks and do not involve human subjects or
sensitive data. As with any automated code-generation system, generated
candidates should be treated as suggestions and validated by task-specific
tests before use, especially outside controlled benchmarks. The method also
relies on access to external LLM APIs, which may affect cost and
reproducibility across users.

\clearpage

\section{Mutation Kernel Prompts}
\label{app:prompts}

\textsc{SMCEvolve} mutates particles through a $2{\times}2$ grid of LLM
prompt templates that varies along two axes:

\begin{table}[h]
\centering
\small
\renewcommand{\arraystretch}{1.15}
\begin{tabularx}{\linewidth}{@{}l X@{}}
\toprule
\textbf{Axis} & \textbf{Options} \\
\midrule
Edit granularity     & \texttt{diff} (local SEARCH/REPLACE patches) \quad vs.\quad \texttt{rewrite} (full-program rewrite) \\
Information source   & \texttt{no\_inspo} (parent only) \quad vs.\quad \texttt{with\_inspo} (parent $+$ top-$k$ and diversity exemplars from the same island) \\
\bottomrule
\end{tabularx}
\end{table}

\noindent
At each call, the system prompt fixes the kernel's role and the expected
output format; the user prompt injects the parent program (rendered with
the language tag \texttt{\{\{language\}\}} into \texttt{\{\{code\_content\}\}}),
its reward (\texttt{\{\{performance\_metrics\}\}}), and --- for
\texttt{with\_inspo} kernels --- a block of reference programs
(\texttt{\{\{inspiration\_section\}\}}). When no reference programs are
available (e.g.\ at initialization), \texttt{with\_inspo} kernels fall
back to their \texttt{no\_inspo} counterpart.

\subsection{\texttt{diff\_no\_inspo}: single-particle local kernel}
\label{app:prompts_diff_no}

\begin{systempromptbox}
You are an expert programmer tasked with making targeted improvements to an existing program.
Focus on small, precise edits that improve performance -- fix inefficiencies, tune parameters, optimize hot paths, or refine logic.

You MUST respond using an edit name, description, and the exact SEARCH/REPLACE diff format:

<NAME>
A shortened name summarizing the edit. Lowercase, no spaces, underscores allowed.
</NAME>

<DESCRIPTION>
Explain the targeted change you are making and why it should improve performance.
</DESCRIPTION>

<DIFF>
<<<<<<< SEARCH
# Original code to find and replace (must match exactly including indentation)
=======
# New replacement code
>>>>>>> REPLACE
</DIFF>

* Every SEARCH section must be copied verbatim from the current file, including indentation, whitespace, and comments -- matching is byte-for-byte.
* Every SEARCH section must match EXACTLY ONE location in the current file. If the snippet you want to change appears more than once, add enough surrounding lines of context to make the match unique; ambiguous edits are rejected.
* SEARCH must differ from REPLACE (no-op edits are rejected).
* You can propose multiple independent edits. SEARCH/REPLACE blocks follow one after another with no other text between them; they are applied in order, so later blocks see the result of earlier ones.
* Make sure the file still runs after your changes.
* IMPORTANT: Do not rewrite the entire program -- focus on targeted, surgical improvements.
\end{systempromptbox}

\begin{userpromptbox}
# Current Program

```{{language}}
{{code_content}}
```

Performance metrics:
{{performance_metrics}}

# Task

Suggest targeted edits to improve the program's performance.
Focus on the most impactful small changes -- parameter tuning, algorithmic micro-optimizations, or bug fixes.
Describe each change with a SEARCH/REPLACE block.
\end{userpromptbox}

\subsection{\texttt{diff\_with\_inspo}: interacting local kernel}
\label{app:prompts_diff_inspo}

\begin{systempromptbox}
You are an expert programmer tasked with making targeted improvements to an existing program.
You will be shown the current program AND one or more high-performing reference programs that solve the same problem.
Study the reference programs to identify specific techniques, parameter choices, or code patterns that could be surgically transplanted into the current program to improve it.

[Output format identical to diff_no_inspo: <NAME>, <DESCRIPTION>, and <DIFF> with SEARCH/REPLACE blocks.]

* IMPORTANT: Do not rewrite the entire program -- borrow and adapt specific ideas from the references via targeted edits.
\end{systempromptbox}

\begin{userpromptbox}
# Current Program

```{{language}}
{{code_content}}
```

Performance metrics:
{{performance_metrics}}

# Reference Programs

The following high-performing programs solve the same problem. Study their techniques and borrow specific ideas to improve the current program via targeted edits.

{{inspiration_section}}

# Task

Identify the most useful techniques from the reference programs and transplant them into the current program via small, targeted SEARCH/REPLACE edits.
Focus on adapting specific patterns -- do not rewrite the entire program.
\end{userpromptbox}

\subsection{\texttt{rewrite\_no\_inspo}: single-particle global kernel}
\label{app:prompts_rewrite_no}

\begin{systempromptbox}
You are an expert algorithm designer. Given a program and its performance, design a fundamentally improved or completely different algorithm to solve the same problem.
Do not just tweak the existing code -- rethink the approach from scratch. Consider different data structures, algorithmic paradigms, mathematical formulations, or optimization strategies.

You MUST respond using a short summary name, description, and the full code:

<NAME>
A shortened name summarizing the code you are proposing. Lowercase, no spaces, underscores allowed.
</NAME>

<DESCRIPTION>
Explain the new algorithmic approach you are taking, how it differs from the current one, and why it should perform better.
</DESCRIPTION>

<CODE>
```{{language}}
# The new program here.
```
</CODE>

* Your program must maintain the same inputs and outputs as the original.
* Think broadly -- consider entirely different algorithms, not just parameter changes.
* Make sure the file still runs after your changes.
\end{systempromptbox}

\begin{userpromptbox}
# Current Program

```{{language}}
{{code_content}}
```

Performance metrics:
{{performance_metrics}}

# Task

Design a new algorithm to replace the current implementation.
Do not make incremental changes -- propose a fundamentally different or significantly improved approach.
Your new program must maintain the same inputs and outputs.
\end{userpromptbox}

\subsection{\texttt{rewrite\_with\_inspo}: interacting global kernel (crossover)}
\label{app:prompts_rewrite_inspo}

\begin{systempromptbox}
You are an expert algorithm designer. You will be shown the current program AND one or more high-performing reference programs that solve the same problem.
Your task is to synthesize a new program that combines the best ideas from all provided programs, or uses them as inspiration to design something even better.
Think of this as intelligent crossover -- not just copy-paste, but creative recombination of algorithmic insights.

[Output format identical to rewrite_no_inspo: <NAME>, <DESCRIPTION>, and full <CODE>.]

* Combine strengths from multiple programs -- don't just pick one and copy it.
\end{systempromptbox}

\begin{userpromptbox}
# Current Program

```{{language}}
{{code_content}}
```

Performance metrics:
{{performance_metrics}}

# Reference Programs

The following high-performing programs solve the same problem. Use them as inspiration to design a new, superior program that combines the best ideas.

{{inspiration_section}}

# Task

Synthesize a new program that combines the best elements from the current program and the reference programs.
Go beyond simple merging -- creatively recombine algorithmic insights to produce something better than any individual program.
Your new program must maintain the same inputs and outputs.
\end{userpromptbox}

\section{Visualizing an \textsc{SMCEvolve} Run}
\label{app:flow}

\Cref{fig:flow} traces a real run of \textsc{SMCEvolve} on the
\emph{Circle Packing in a Rectangle} ($N{=}21$) task with $I{=}2$ islands,
$N{=}8$ particles per island, and $K{=}2$ MH proposals per particle.
Each panel shows one SMC iteration on a single island, organized into the
three sub-steps of the kernel: \emph{reweight} (top), \emph{resample}
(middle), and \emph{mutate} (bottom). In the reweight row, node area is
proportional to the normalized importance weight $w_i/\max_j w_j$, and a
thick colored ring marks particles that arrived via inter-island migration
in the previous epoch. Resample arrows depict the multinomial replication
map: arrows converging on a single parent indicate replication of a
high-weight ancestor, while parents with no outgoing arrow are discarded.
The mutate block stacks one row per MH proposal, with vertical dashed lines
linking the successive proposals of the same particle.

Node \emph{fill} color encodes reward through a global red--yellow--green
colormap normalized over the full run. On mutate nodes we additionally
encode three pieces of kernel metadata: \emph{stroke color} distinguishes
the edit operator (blue $=$ local diff edit, orange $=$ full-program
rewrite); \emph{stroke width} marks whether the proposal was conditioned on
inspiration programs (thick) or not (thin); and \emph{line style with fill
opacity} encodes the MH outcome (solid/opaque $=$ accepted,
dashed/translucent $=$ rejected with the fill reverting to the parent
reward, dotted/faded $=$ skipped). The header above each panel reports the
annealing schedule statistics ($\lambda_t$, $\Delta\beta_t$, $\beta_t$,
ESS at $\lambda_t$) together with the best and mean reward of the
population, so that the kernel-induced dynamics can be cross-referenced
with the adaptive temperature schedule at a glance.

\begin{figure}[t]
  \centering
  \includegraphics[width=\textwidth]{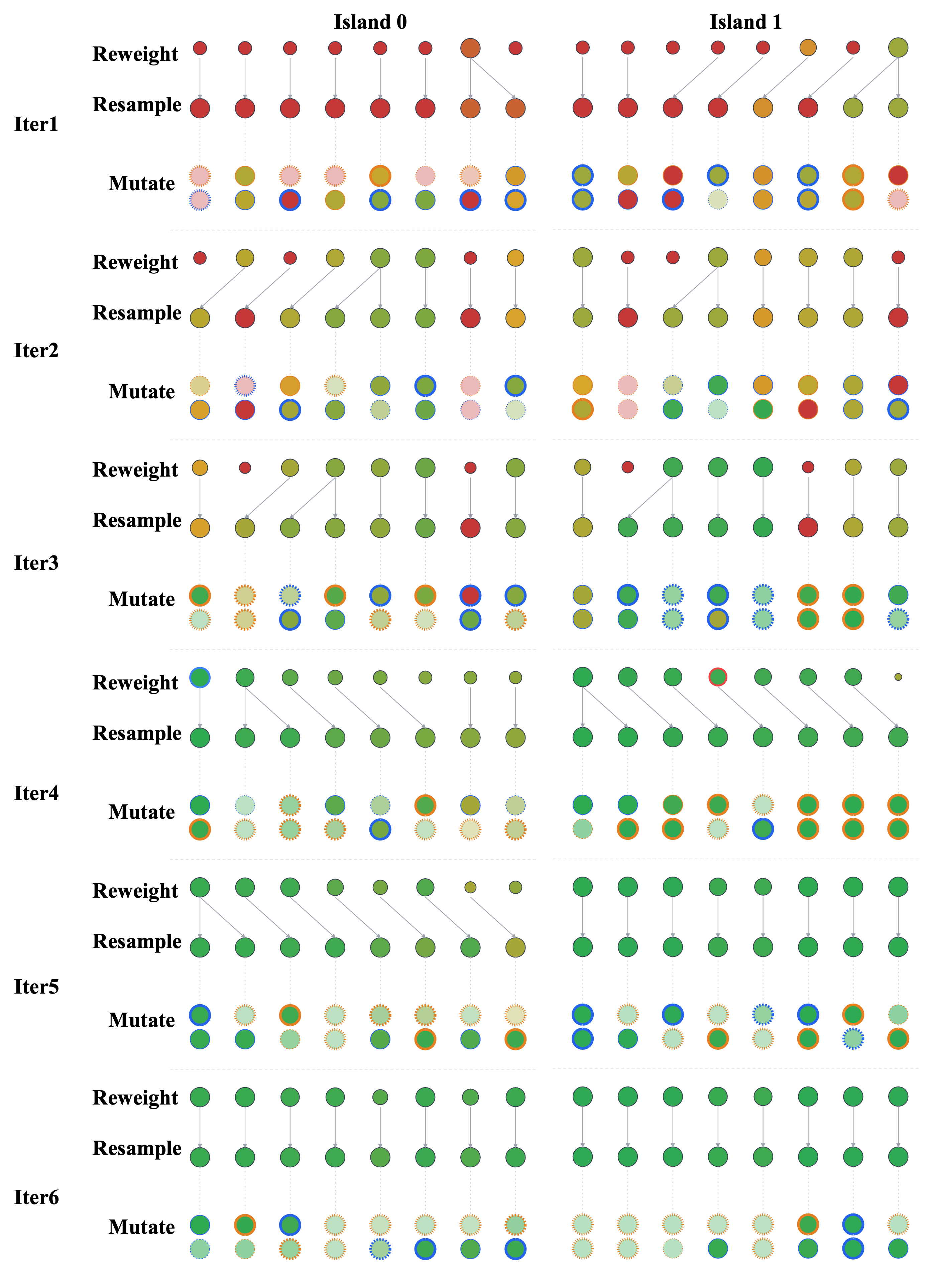}
  \caption{\textbf{Per-iteration SMC flow on Circle Packing $N{=}21$}
  ($I{=}2$ islands, $N{=}8$ particles per island, $K{=}2$ MH proposals).
  Rows of each panel correspond to \emph{reweight}, \emph{resample}, and
  \emph{mutate}; node fill encodes reward (red--yellow--green); thick rings
  in the reweight row mark inter-island migrants; resample arrows show the
  multinomial replication map; on mutate nodes, stroke color distinguishes
  the edit operator (blue $=$ local diff, orange $=$ full rewrite), stroke
  width flags inspiration-conditioned proposals, and line style/opacity
  encodes the MH outcome (solid $=$ accepted, dashed $=$ rejected, dotted
  $=$ skipped). Panel headers list the annealing schedule
  ($\lambda_t,\Delta\beta_t,\beta_t,\mathrm{ESS}$) and the population
  best/mean reward.}
  \label{fig:flow}
\end{figure}

\clearpage

%% file: sections/6_related_work.tex
\section{Related Work}
\label{sec:related}

\paragraph{LLM-driven program evolution and scientific discovery.}
LLMs have become effective generators inside evaluator-guided discovery loops.
FunSearch repeatedly proposes and evaluates programs to obtain mathematical
constructions~\citep{romera2024mathematical}; AlphaEvolve scales this idea to a
general coding agent for scientific and algorithmic discovery
~\citep{novikov2025alphaevolve}; and ShinkaEvolve, ReEvo, and Evolution of
Heuristics improve sample efficiency through open-ended evolution, reflection,
and heuristic design~\citep{lange2025shinkaevolve,ye2024reevo,liu2024evolution}.
Related agent benchmarks and systems study automated scientific workflows,
machine-learning engineering, symbolic equation discovery, materials discovery,
and numerical-program optimization
~\citep{yamada2025ai,chan2024mle,imajuku2025ale,shojaee2024llm,shojaee2025llm,abhyankarllema,press2025}.
These works motivate evaluator-guided program search, but their parent
selection, population management, mutation schedules, and stopping rules are
typically engineered as separate heuristics. \textsc{SMCEvolve} focuses on this
population-search core and interprets it as SMC sampling from a reward-tilted
program distribution.

\paragraph{SMC for language-model inference and reasoning.}
Recent work also applies SMC directly to language-model generation. Lew et al.
use SMC steering for constrained generation specified as posterior inference in
probabilistic sequence models~\citep{lew2023smc}. Zhao et al. develop twisted
SMC for sampling from sequence-level unnormalized targets in LLM inference and
safety tasks~\citep{zhao2024twisted}, and Feng et al. use twisted SMC to guide
multi-step mathematical reasoning~\citep{feng2025step}. These methods treat
prefixes, complete sequences, or reasoning traces as particles within a single
generation episode. In contrast, \textsc{SMCEvolve} treats executable programs
as particles, obtains rewards from an external evaluator, and uses SMC to
structure population-based scientific program evolution rather than controlled
decoding.

\paragraph{SMC, annealing, and adaptive sampling foundations.}
Our method builds on SMC samplers, which move particles through intermediate
distributions via weighting, resampling, and mutation
~\citep{del2006sequential,dai2022invitation,delmoral2004feynman,chopin2020introduction}.
The reward-tempered path is related to annealed importance sampling, simulated
annealing, and cross-entropy methods for stochastic optimization
~\citep{neal2001annealed,kirkpatrick1983optimization,rubinstein1999crossentropy}.
We import finite-sample SMC complexity theory~\citep{marion2023finite} rather
than proving a new general theorem. The contribution is the specialization to
LLM-driven program evolution: the reward-tilted target, bridge path, resampling
weights, and reference-kernel requirements jointly explain the design levers and
limits of the evolutionary coding agent.